\documentclass{article}

\PassOptionsToPackage{numbers,sort&compress}{natbib}
\usepackage[preprint]{neurips_2026}

\usepackage[T1]{fontenc}
\usepackage[utf8]{inputenc}

\usepackage{amsmath,amssymb,amsthm}
\usepackage{xcolor}
\usepackage{booktabs}
\usepackage{float}
\usepackage{graphicx}
\usepackage{algorithm}
\usepackage{algpseudocode}
\usepackage{multirow}
\usepackage[hidelinks]{hyperref}
\hypersetup{
  pdftitle={SHiPPO: Recurrent Memory with Transported Polynomial Projections},
  pdfauthor={Tomoya Mizuguchi and Bum Jun Kim},
  pdfsubject={},
  pdfkeywords={state space models, recurrent memory, HiPPO, structured sequence models, transported memory}
}
\usepackage{placeins}
\usepackage{array}
\usepackage{ragged2e}
\usepackage{makecell}

\newcolumntype{L}[1]{>{\RaggedRight\arraybackslash}p{#1}}
\newcolumntype{C}[1]{>{\Centering\arraybackslash}p{#1}}
\newcommand{\ms}[2]{\makecell[c]{$#1$\\[-0.25ex]{\scriptsize $\pm\,#2$}}}
\theoremstyle{definition}
\newtheorem{assumption}{Assumption}[section]
\newtheorem{definition}[assumption]{Definition}
\newtheorem{example}[assumption]{Example}
\newtheorem*{assumption*}{Assumption}
\newtheorem*{definition*}{Definition}
\newtheorem*{example*}{Example}

\theoremstyle{plain}
\newtheorem{lemma}[assumption]{Lemma}
\newtheorem{proposition}[assumption]{Proposition}
\newtheorem{theorem}[assumption]{Theorem}
\newtheorem{corollary}[assumption]{Corollary}
\newtheorem*{lemma*}{Lemma}
\newtheorem*{proposition*}{Proposition}
\newtheorem*{theorem*}{Theorem}
\newtheorem*{corollary*}{Corollary}

\theoremstyle{remark}
\newtheorem{remark}[assumption]{Remark}
\newtheorem*{remark*}{Remark}

\title{SHiPPO: Recurrent Memory with Transported Polynomial Projections}

\author{%
  Tomoya Mizuguchi\\
  Kyoto University, Japan\\
  \And
  Bum Jun Kim\thanks{Corresponding author: \texttt{bumjun.kim@weblab.t.u-tokyo.ac.jp}}\\
  The University of Tokyo, Japan
}

\begin{document}

\maketitle
\enlargethispage{2\baselineskip}

\begin{abstract}
HiPPO gives recurrent states memory semantics as coefficients of online
polynomial projections, but in fixed channel coordinates. Modern selective SSMs,
by contrast, rely on token-dependent control and channel interaction. We
introduce SHiPPO (Sylvester HiPPO), a transported projection-memory prior that
lifts HiPPO coefficient memories into a moving channel frame. For any fixed or
realized right-transport path, SHiPPO transports the approximation family and
channel metric together; conditional on that path, the state is ordinary HiPPO
in a tied moving frame and follows Sylvester coefficient dynamics, preserving
the left online-memory operator while adding right-action transport. For
selective-SSM execution, we derive a restricted group-local realization with
controller-compatible right actions, exponential-adjusted updates, exact
block-affine scan, and recurrent decoding. We also give a
simultaneous-reducibility criterion identifying when right transports collapse
to static mixing plus independent scalar or blockwise banks. Controlled
diagnostics show that larger current-token write rank improves ordinary
prediction error but cannot recover order-sensitive changes to already-written
memory; transported-memory variants recover this signal, which disappears when
the transport pathway is removed. A finite-field associative-recall diagnostic
with interleaved bindings, operations, and queries provides complementary
autoregressive evidence while leaving the preferred right-action realization
open. Taken together, these results support SHiPPO as a mechanistically
grounded transported-memory prior, with evidence focused on memory mechanisms
rather than broad sequence-modeling dominance.
\end{abstract}

\vspace{-0.8em}


\section{Introduction}
\label{sec:intro}

State space models (SSMs) have re-emerged as efficient sequence backbones,
combining recurrent inference with long-context scalability. We frame recurrent
sequence modeling as \emph{inductive-bias selection}: choosing an ODE,
recurrence family, parameterization, discretization, or initialization selects a
structured family of causal dynamics and favors particular spectra, timescales,
and optimization paths \citep{mitchell1980need,gordon1995bias,vardi2023implicit}.
For recurrent memory, a stronger model-based bias is possible: the hidden state
can be designed to represent an approximation of the revealed history rather
than an unconstrained latent vector \citep{vonrueden2023informed,shlezinger2023modelbased}.
The Legendre Memory Unit and HiPPO are canonical examples: LMU represents a
sliding history window in a Legendre basis, while HiPPO derives online
compression dynamics from projection onto polynomial bases
\citep{voelker2019lmu,gu2020hippo}. Throughout, we use ``prior'' in this
inductive- or modeling-bias sense, not necessarily as a Bayesian prior over
parameters \citep{fortuin2022priors}.

In HiPPO, the recurrent ODE is not merely an architectural template; it is the
coefficient dynamics induced by an online approximation problem, so the state
stores coefficients of a history approximation \citep{gu2020hippo}. This
objective-derived view has shaped structured state-space models, including S4,
generalized-basis interpretations of HiPPO, and diagonal simplifications such as
DSS and S4D \citep{gu2021s4,gu2022trainhippo,gupta2022dss,gu2022s4d}.
Recent work further studies how memory priors interact with initialization,
timescales, diagonalization, robust parameterization, and spectral or frequency
bias \citep{liu2025autocorrelation,lienen2025unhippo,yu2023ptd,yu2024hope,yu2024frequency,solozabal2025spectralbias}.
Our goal is to extend the underlying principle: an online approximation
objective should justify the recurrent memory dynamics, while parameterization
and initialization determine how a trainable model explores the resulting
family.

Modern SSM and linear-recurrence blocks have moved beyond the original
one-sided independent-memory setting. S5 replaces many independent SISO SSMs by
a MIMO state-space layer \citep{smith2023s5}. H3 introduces multiplicative
interactions between SSM outputs and input projections for language modeling
\citep{fu2023h3}. Mamba makes SSM parameters input-dependent and pairs the
resulting selective recurrence with a hardware-aware scan \citep{gu2023mamba}.
SSD/Mamba-2 analyzes selective recurrences through semiseparable structure and
efficient algorithms \citep{dao2024transformersssms}. Recent gated
linear-attention and linear-recurrence models combine gates, delta-style
updates, or state expansion for adaptive memory control
\citep{yang2024gla,yang2024deltanet,qin2024hgrn2,yang2025gateddeltanet}, and
recent variants such as Mamba-3 further broaden the expressive recurrence design
space \citep{lahoti2026mamba3}. These developments make clear that channel
interaction and token-dependent control are not themselves
the missing novelty. The open question for memory-prior design is different: can
channel interaction be made part of the online-approximation memory semantics
itself, rather than added only as an architectural mixer, projection, gate, or
controller?

We introduce \emph{SHiPPO} (\emph{Sylvester HiPPO}), a transported online
approximation framework for channel-interacting recurrent memory. At the
operator level, SHiPPO is defined by a projection problem, not by postulating a
recurrence in isolation. Given an online-memory basis and an admissible
right-transport path, SHiPPO jointly transports the channel metric and the
approximation family. The resulting coefficient matrix satisfies the Sylvester
dynamics
\[
    \dot C(t)=A_L(t)C(t)+B_L(t)f(t)^\top+C(t)A_R(t).
\]
Here \(A_L,B_L\) are supplied by the chosen online approximation problem, while
\(A_R\) describes how channel coordinates are transported along the history. The
inner projection problem minimizes only over the coefficient state \(C\); the
transport path may be fixed, parameterized, or generated by a causal controller
in an outer trainable model. Thus, when \(A_R\) is input-dependent, the semantics
is pathwise: conditional on the realized transport path, the state is the
coefficient of the corresponding transported approximation problem, even though
the full input--output map may be nonlinear. For a fixed realized path, SHiPPO is
ordinary HiPPO in a tied moving channel frame, with the history encoder,
coefficient decoder, and Sylvester gauge term determined by the same transport
path. 

This gives a lifting principle for deep SSM layers whose dynamics, structure, or
initialization are motivated by online-memory priors. The lift does not prescribe
a universal right generator or a particular right-transport initialization.
Rather, it replaces a one-sided online-memory coefficient equation by a
transported coefficient equation, separating the memory basis supplied by the
left operator from the channel-frame evolution supplied by the right transport.
For S4/S4D-style models, this corresponds to adding a right-action transport
term to an existing coefficient equation. For Mamba-style selective SSMs, the
constraint is stronger: the layer must preserve tokenwise parameter generation,
a lightweight exact scan, and recurrent decoding. A fully channelwise lift of a
diagonal selective recurrence does not, in general, preserve a small scan algebra
under nontrivial right transport. We therefore study a restricted
scan-compatible SHiPPO-derived realization: channels are partitioned into small
transport groups, the left dynamics remain diagonal in the state dimension and
tied within each group, and the right transport is controller-compatible and
group-local. These restrictions are not part of the abstract SHiPPO definition;
they are the computational price paid to obtain exponential-adjusted updates,
exact group-local block-affine scan, and recurrent decoding for the chosen right
action. We also identify a simultaneous-reducibility collapse criterion showing
when right transports reduce, up to static channel mixing, to independent scalar
or blockwise transported banks. This motivates non-reducible transport
parameterizations in the scan-compatible realization.

This distinction also has empirical content. High-rank current-step source/write
updates can increase what is written into memory at the current token, but they
do not by themselves implement a later right action on memory that has already
been written. We therefore use paired noncommutative transport diagnostics to
separate source/write rank from future transport of stored memory, rather than
treating broad associative recall as the main evidence. We then use
Transport-MQAR as a complementary autoregressive finite-field diagnostic in
which bindings, operations, and queries are interleaved.

\begin{figure}[t]
\centering
\includegraphics[width=0.95\linewidth]{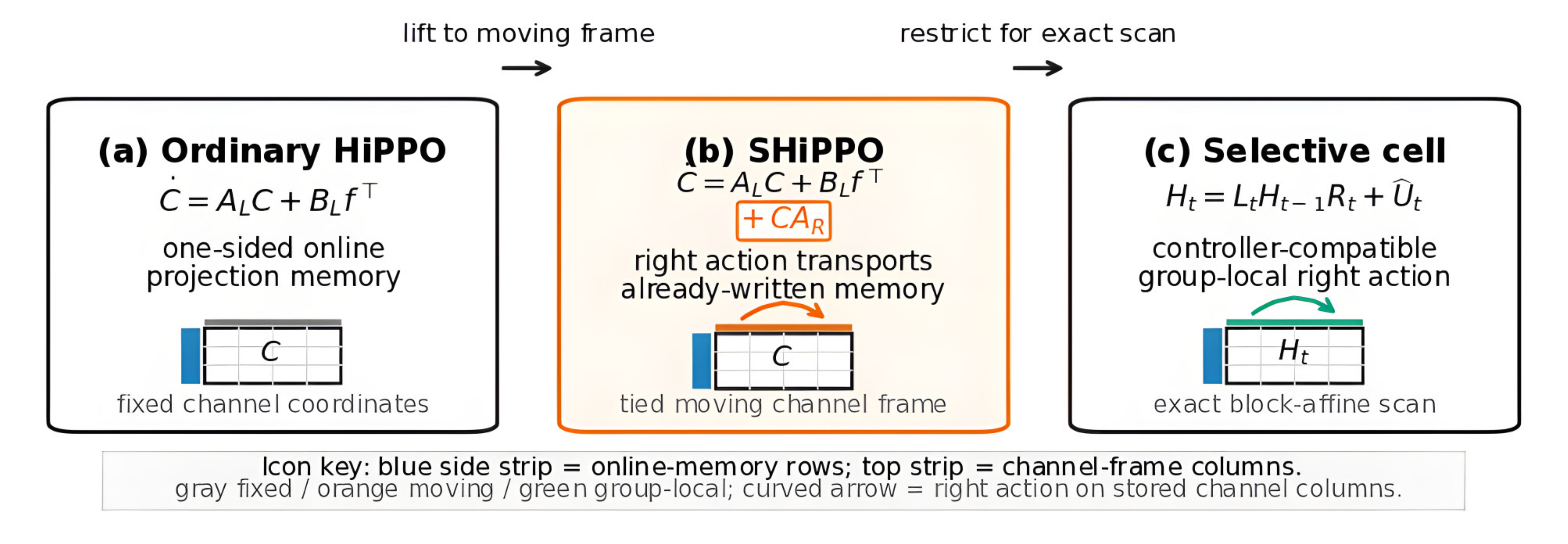}
\caption{Overview of the SHiPPO lift. Ordinary HiPPO gives one-sided online
projection memory in fixed channel coordinates. SHiPPO transports the channel
frame while preserving the left online-memory operator, yielding the right-action
term \(CA_R\) on stored memory. The selective cell is a scan-compatible
restriction with controller-compatible group-local right transport and exact
block-affine scan.}
\label{fig:shippo-overview}
\end{figure}

\paragraph{Contributions.}
(i) We formulate SHiPPO as a transported online approximation problem for
channel-interacting memory and prove that coupling the transported approximation
family with the transported channel metric yields Sylvester coefficient dynamics.
(ii) We derive a pathwise lifting principle for online-memory coefficient
equations: a one-sided projection memory is lifted by keeping the left
online-memory operator fixed and adding a right-action transport path as an
external modeling or learning choice.
(iii) We instantiate this prior in a restricted scan-compatible selective SSM
cell with group-tied diagonal-left dynamics and controller-compatible
group-local right transport, deriving exponential-adjusted discrete updates and
proving exact group-local block-affine scan closure for the resulting
recurrence.
(iv) We identify a simultaneous-reducibility collapse criterion under which
right transports are equivalent, up to static channel mixing, to independent
scalar or blockwise transported banks; this motivates non-reducible split-flow
transport designs.
(v) We provide paired noncommutative diagnostics separating high-rank
source/write updates from future right transport of already-written memory,
showing that source rank and transported memory are distinct mechanisms rather
than interchangeable forms of channel interaction.
(vi) We evaluate scan-compatible SHiPPO realizations on Transport-MQAR and a
controller-suffix intervention as complementary autoregressive diagnostics of
learned right-action pathways.


\section{SHiPPO (Sylvester HiPPO): Transported Online Projection Memories}
\label{sec:shippo}

SHiPPO (\emph{Sylvester HiPPO}) builds on the HiPPO framework of online polynomial projection operators \citep{gu2020hippo,gu2022trainhippo} by adding a chosen right-transport path to the channel frame.
We define SHiPPO as an online approximation problem for that realized transport path, rather than by postulating a recurrence in isolation.
Throughout this section, fix an admissible path \(A_R\in L^1([0,T];\mathbb R^{d\times d})\); all identities are conditional on this realized path, and we suppress this dependence in the notation.
If \(A_R\) is generated by a causal controller in a trainable model, the same projection semantics and coefficient dynamics hold \emph{pathwise} for each realized trajectory.
The inner optimization is always over the coefficient matrix \(C\), not over the transport path.

Let \(f:[0,\infty)\to\mathbb R^d\) be a \(d\)-channel signal.
For each \(t>0\), let \(\mu_t\) be a measure on \([0,t]\), and let \(\Phi_t:[0,t]\to\mathbb R^N\) be a basis with invertible Gram matrix
\[
G(t):=\int_0^t \Phi_t(\tau)\Phi_t(\tau)^\top\,d\mu_t(\tau).
\]
When \(d\mu_t(\tau)=w_t(\tau)d\tau\), define
\[
\psi(t,\tau):=w_t(\tau)\Phi_t(\tau),
\qquad
\partial_t\psi(t,\tau)=A_L(t)\psi(t,\tau)\quad(\tau<t),
\qquad
B_L(t):=\psi(t,t),
\]
where the middle identity is the usual HiPPO closure condition.
Ordinary vector-valued HiPPO approximates \(f|_{[0,t]}\) by \(C^\top\Phi_t(\tau)\) and has coefficient matrix \(C_H(t)\) satisfying
\[
G(t)C_H(t)
=
\int_0^t \Phi_t(\tau)f(\tau)^\top\,d\mu_t(\tau).
\]
In the orthonormalized case \(G(t)=I_N\), this gives the closed one-sided coefficient dynamics
\[
\dot C_H(t)=A_L(t)C_H(t)+B_L(t)f(t)^\top .
\]
Appendix~\ref{app:ordinary-hippo} recalls the variational derivation and the non-orthonormal form.

We now add right transport.
Let \(P(t,\tau)\in GL(d)\) be the state-transition family for the right generator,
\[
\partial_tP(t,\tau)=P(t,\tau)A_R(t),
\qquad
P(\tau,\tau)=I_d,
\]
which exists and is invertible under the integrability assumption on \(A_R\) \citep{coddington1955theory}.
Set
\[
M_P(t,\tau):=P(t,\tau)P(t,\tau)^\top,
\qquad
\|u\|_{M_P(t,\tau)}^2:=u^\top M_P(t,\tau)u.
\]

\begin{definition}[SHiPPO online approximation problem]
\label{def:shippo-projection}
For each \(t>0\), define the transported approximation family
\[
\mathcal G_t^{\mathrm{SH}}
=
\{\,\tau\mapsto P(t,\tau)^{-\top}C^\top\Phi_t(\tau):
C\in\mathbb R^{N\times d}\,\}.
\]
The SHiPPO coefficient matrix is
\[
C_S(t)
=
\arg\min_{C\in\mathbb R^{N\times d}}
\int_0^t
\left\|
 f(\tau)-P(t,\tau)^{-\top}C^\top\Phi_t(\tau)
\right\|_{M_P(t,\tau)}^2
\,d\mu_t(\tau).
\]
We write \(C_S(t)=\operatorname{shippo}_t(f)\).
\end{definition}

Definition~\ref{def:shippo-projection} transports the approximation family and the channel metric together.
This coupling is essential: transporting only the channel metric while keeping the ordinary HiPPO family generally destroys finite HiPPO-style coefficient closure when the metric depends on \(\tau\).
Appendix~\ref{app:metric-only} gives the stationary calculation and the \(\tau\)-independent special case.

The coupled construction is conjugate to ordinary HiPPO.
For fixed \(t\), define
\[
(\mathcal T_tu)(\tau):=P(t,\tau)^\top u(\tau).
\]
Then \(\mathcal T_t\) is an isometry from the SHiPPO metric to the Euclidean HiPPO metric and maps \(\mathcal G_t^{\mathrm{SH}}\) bijectively onto the ordinary HiPPO family:
\[
\int_0^t u(\tau)^\top M_P(t,\tau)v(\tau)\,d\mu_t(\tau)
=
\int_0^t (\mathcal T_tu)(\tau)^\top(\mathcal T_tv)(\tau)\,d\mu_t(\tau),
\qquad
\operatorname{shippo}_t(f)=\operatorname{hippo}_t(\mathcal T_tf).
\]
The final equality is an identity of coefficient matrices; Appendix~\ref{app:shippo-conjugacy} gives the full argument.

\begin{theorem}[Normal equation and Sylvester coefficient dynamics]
\label{thm:shippo-dynamics}
The SHiPPO coefficient matrix satisfies
\[
G(t)C_S(t)
=
\int_0^t \Phi_t(\tau)f(\tau)^\top P(t,\tau)\,d\mu_t(\tau).
\]
If, in addition, \(G(t)=I_N\), \(d\mu_t(\tau)=w_t(\tau)d\tau\), the HiPPO closure condition above holds, and the usual Leibniz-rule regularity assumptions are satisfied, then, at differentiability times,
\[
C_S(t)
=
\int_0^t \psi(t,\tau)f(\tau)^\top P(t,\tau)\,d\tau
\]
and
\[
\dot C_S(t)
=
A_L(t)C_S(t)+B_L(t)f(t)^\top+C_S(t)A_R(t).
\]
\end{theorem}

\begin{proof}[Proof sketch]
By the conjugacy above, \(C_S(t)\) is the ordinary HiPPO coefficient matrix of the transformed signal \((\mathcal T_tf)(\tau)=P(t,\tau)^\top f(\tau)\), which gives the normal equation.
In the orthonormalized case, differentiate the integral representation using Leibniz' rule, \(\partial_t\psi=A_L\psi\), \(\partial_tP=PA_R\), and \(P(t,t)=I_d\).
The boundary term gives \(B_Lf^\top\), while the two interior derivatives give \(A_LC_S\) and \(C_SA_R\).
Full regularity and first-variation details are in Appendix~\ref{app:shippo-theorem}.
\end{proof}

In the orthonormalized case, the coefficient ODE has the standard differential Sylvester form \citep{behr2019differential,simoncini2016computational}.
Theorem~\ref{thm:shippo-dynamics} therefore shows that SHiPPO is not an arbitrary modification of a state equation: the left operator \(A_L,B_L\) is inherited from the underlying online approximation problem, while the right action enters only through the chosen transport path.
This is the sense in which SHiPPO lifts one-sided online-memory dynamics to channel-interacting memory.

\begin{corollary}[Pathwise transported lift of closed online-memory dynamics]
\label{cor:shippo-lift}
Any closed one-sided online-memory coefficient equation arising from the projection setup above,
\[
\dot C(t)=A_L(t)C(t)+B_L(t)f(t)^\top,
\]
admits, for any admissible right-transport path \(A_R\), the pathwise transported dynamics
\[
\dot C_S(t)
=
A_L(t)C_S(t)+B_L(t)f(t)^\top+C_S(t)A_R(t).
\]
The left operator is inherited from the underlying online approximation problem, whereas the right transport is an external modeling choice.
Architectural restrictions on \(A_R\) are introduced only in the scan-compatible realization of Section~\ref{sec:selective-transport-cell}, not in the abstract SHiPPO definition.
\end{corollary}

SHiPPO reduces exactly to ordinary HiPPO when \(A_R\equiv0\), equivalently
\(P(t,\tau)\equiv I_d\).
For a nontrivial realized transport path, SHiPPO is ordinary HiPPO in a tied
moving channel frame, not a static channel mixer or an arbitrary encoder--decoder
factorization.
Appendix~\ref{app:reductions-moving-frame} gives the exact moving-frame
factorization and non-reduction statements.
When \(A_R\) is input-dependent, the same interpretation remains valid pathwise,
but it does not yield a fixed input-independent encoder--HiPPO--decoder reduction.


\section{A Scan-Compatible SHiPPO Lift for Mamba-Style Selective SSMs}
\label{sec:selective-transport-cell}

Section~\ref{sec:shippo} defines SHiPPO as an operator-level transported projection memory. A selective SSM layer cannot use this full operator without additional computational restrictions. We therefore study one scan-compatible SHiPPO-derived realization: channels are partitioned into small transport groups, the left dynamics are diagonal in the state dimension and tied across channels within each group, and the right transport is group-local and controller-compatible. These restrictions are not part of the abstract SHiPPO definition; they are the price of exact scan and recurrent decoding.

Let $D$ denote the inner channel width of the selective branch, and partition the channels into $G$ groups of width $P$, so $D=GP$. We write formulas for one group and omit the group index. Thus the group-local input is $x_t\in\mathbb R^P$, the memory state is $H_t\in\mathbb R^{N\times P}$, and the right transport acts by $R_t\in GL(P)$. The ungrouped case is recovered by taking $P=D$.
A fully channelwise lift of a diagonal selective recurrence generally requires row-dependent right transition products and does not close in the small block-affine scan algebra used below; Appendix~\ref{app:direct-lift-obstruction} gives the two-step obstruction. The realization in this section should therefore be read as a restricted Mamba-style cell derived from the SHiPPO prior, not as the full operator-level SHiPPO construction and not as the original per-channel Mamba parameterization.

\subsection{From the SHiPPO prior to a scan-friendly selective recurrence}
\label{sec:scan-friendly-selective-recurrence}

Our restricted continuous-time selective lift is
\begin{equation}
\dot H(t)
=
\operatorname{Diag}(a(t))\,H(t)+b(t)x(t)^\top+H(t)A_R(t),
\label{eq:group-tied-shippo-ct}
\end{equation}
where $a(t)\in\mathbb R^N$, $b(t)\in\mathbb R^N$, $x(t)\in\mathbb R^P$, and $A_R(t)\in\mathbb R^{P\times P}$. The left dynamics remain diagonal in the state dimension, as in diagonal selective SSMs, but are tied across the $P$ channels in the transport group. The factorized source $b(t)x(t)^\top$ keeps the HiPPO-like interpretation that the left operator injects the current input into a temporal memory basis. More general additive sources $U(t)\in\mathbb R^{N\times P}$ are compatible with the same scan algebra; the factorized form is the lightweight selective-cell instance used below.

Equation~\eqref{eq:group-tied-shippo-ct} is Mamba-style only in this restricted sense. It preserves diagonal state-timescale dynamics on the left, while the SHiPPO right action supplies channel interaction inside the memory state itself. The right-transport family is a modeling restrictive bias; preferential choices such as analytic initialization are orthogonal to the scan-compatible realization.

\subsection{Controller-compatible transport and block-affine scan closure}
\label{sec:controller-compatible-transport}

The right transport may be token-dependent, but an exact finite scan summary requires it to be fixed with respect to the main memory variable being scanned. We encode this through a causal controller path $\xi_{1:T}$, computed either input-only, $\xi_t=\phi(x_t)$, or by an auxiliary causal module independent of the main memory recurrence. Conditional on this path, the main recurrence is a sequence of affine two-sided maps.

\begin{definition}[Controller-compatible right transport]
\label{def:controller-compatible-transport}
A right transport $R_t\in GL(P)$ is \emph{controller-compatible} if, conditional on a precomputed causal controller path $\xi_{1:T}$,
\[
R_t=R(\xi_t)
\]
and $R_t$ is independent of the main memory state $H_{t-1}$. A transport of the form $R_t=R(\xi_t,H_{t-1})$ is state-coupled and is not controller-compatible unless the dependence on $H_{t-1}$ is degenerate.
\end{definition}

\begin{proposition}[Exact block-affine scan closure]
\label{prop:affine-closure}
Consider the recurrence
\[
H_t=L_tH_{t-1}R_t+U_t,
\]
where $L_t\in\mathbb R^{N\times N}$, $R_t\in GL(P)$, and $U_t\in\mathbb R^{N\times P}$ are fixed conditional on a precomputed controller path. Then each step is affine in $H_{t-1}$, and summaries compose as
\[
(L_2,R_2,U_2)\star(L_1,R_1,U_1)
=
(L_2L_1,\;R_1R_2,\;L_2U_1R_2+U_2).
\]
Thus the recurrence admits an exact block-affine prefix scan. If $R_t$ depends directly on $H_{t-1}$, the step map is generically nonlinear in $H_{t-1}$, so this finite affine summary algebra is not closed without augmenting the state or imposing special degeneracies.
\end{proposition}

This is an associative prefix-scan computation \citep{blelloch1990prefix}; related parallel scans have been used to parallelize linear recurrent networks over sequence length \citep{martin2018parallelizing}. Appendix~\ref{app:scan-friendly-shippo} gives the proof and a concrete state-coupled failure mode. The scan above is not the original elementwise selective-scan algebra of Mamba. The original one-channel scan is recovered in the trivial limit $P=1$. If the right transport is diagonal or identity, the recurrence decouples across channels within each transport group, although the group-tied left restriction remains. In the nontrivial transported case, exact scan is retained as a group-local block-affine scan.

\subsection{Degenerate versus nontrivial transport families}
\label{sec:block-diagonal-transport-collapse}

Controller compatibility preserves computation, but it does not guarantee expressive transport. A common source of misleading adaptivity is token-dependent coefficients over a generator family that is simultaneously reducible in a fixed channel basis.

\begin{proposition}[Collapse under fixed simultaneous block reduction]
\label{prop:collapse}
Suppose all right-generator basis matrices preserve the same nontrivial block decomposition: for every generator $G_m$ there is a fixed $Q\in GL(P)$ such that
\[
G_m=Q\Lambda_mQ^{-1},
\]
where all $\Lambda_m$ are diagonal or block-diagonal with the same fixed nontrivial block partition. If
\[
A_{R,t}=\sum_m \rho_{t,m}G_m,
\]
then every dense exponential $R_t=\exp(\Delta_tA_{R,t})$ and every fixed-order split product formed from these generators is diagonal or block-diagonal in the same fixed basis. Consequently, the change of variables $\widetilde H_t=H_tQ$ decomposes
\[
H_t=L_tH_{t-1}R_t+U_t
\]
into independent scalar or blockwise transported banks, up to a static change of channel coordinates. If $U_t=b_tx_t^\top$, then $U_tQ=b_t(Q^\top x_t)^\top$, so the same fixed basis change can be absorbed into the group input coordinates.
\end{proposition}

This proposition is a sufficient collapse criterion, not a complete classification of all degeneracies. It rules out the case where token-dependent coefficients only move within a simultaneously reducible generator family. Noncommutativity alone is insufficient: noncommuting generators may still preserve a common block decomposition. The scan-compatible realization should therefore use right-action families whose realized transports admit no fixed nontrivial common block decomposition, if the goal is genuinely channel-interacting transported memory.


\subsection{Controlled split-flow transport and exponential-adjusted discrete cell}
\label{sec:discrete-shippo-cell}

For implementation, each token first fixes a group-local right action \(R_t\),
either by a dense frozen-ODE exponential or by a fixed-order split product of
structured factors. We use the structured generator family
\begin{equation}
A_{R,t}
=
-\operatorname{Diag}(d_t)
+\sum_m \theta_{t,m}\Omega_m
+\sum_n \eta_{t,n}N_n
+\sum_\ell \zeta_{t,\ell}p_\ell q_\ell^\top,
\qquad d_t\in\mathbb R_+^P,
\label{eq:split-flow-generator}
\end{equation}
where \(\Omega_m^\top=-\Omega_m\), \(N_n^2=0\), and optional low-rank terms add
additional channel interaction. The dense backend sets
\(R_t=\exp(\Delta_tA_{R,t})\); a split backend instead implements a fixed-order
product of cheap factor exponentials. Appendix~\ref{app:scan-friendly-shippo}
gives the factor actions, split-product error, and dense-backend local
quadrature analysis. The scan algebra below is exact for whichever discrete
right action \(R_t\) is implemented.

Let \(\Delta_t>0\), \(\lambda_t\in[0,1]\), and \(U_t=b_tx_t^\top\). We use
\begin{align}
L_t&:=\exp\!\bigl(\Delta_t\operatorname{Diag}(a_t)\bigr),
\qquad
R_t:=
\begin{cases}
\exp(\Delta_tA_{R,t}), & \text{dense backend},\\
R_t^{\mathrm{split}}, & \text{split backend},
\end{cases}
\nonumber\\
\widehat U_t
&:=(1-\lambda_t)\Delta_t L_tU_{t-1}R_t+
\lambda_t\Delta_t U_t,
\qquad
H_t:=L_tH_{t-1}R_t+\widehat U_t .
\label{eq:discrete-shippo-cell}
\end{align}
For the first step, implementations may set the boundary source \(U_0=0\).

\begin{theorem}[Exponential-adjusted SHiPPO cell]
\label{thm:discrete-cell}
For the dense backend \(R_t=\exp(\Delta_tA_{R,t})\), the update
\eqref{eq:discrete-shippo-cell} is a two-point exponential source
discretization of \eqref{eq:group-tied-shippo-ct} under step-frozen \(a_t\) and
\(A_{R,t}\). For a split backend, the same formula defines the exact discrete
recurrence for the implemented split right action. If \(A_{R,t}\equiv0\), then
\(R_t=I_P\) and the cell reduces to the corresponding one-sided diagonal-left
selective update; with \(P=1\) and \(\lambda_t=1\), this recovers the
corresponding one-channel exponential-Euler selective update.
\end{theorem}

Thus nontrivial right transport does not preserve the original elementwise
Mamba scan. Instead, the restricted SHiPPO realization yields an exact
group-local block-affine scan, with computation controlled by the group width
\(P\) and with the right action interpreted as learned transport of memory
coordinates rather than as a generic channel mixer.

\section{Experiments: Separating Source Writes from Right Transport}
\label{sec:experiments}

We evaluate SHiPPO as a transported-memory prior. The experiments ask a
mechanistic question: can high-rank current-step source/write updates substitute
for future right transport of memory that has already been written? We first
use a paired noncommutative diagnostic that blocks payload-dependent source
injection at operation tokens. We then use Transport-MQAR as a complementary
autoregressive recall diagnostic.

The no-right source/write baseline removes the right action,
\begin{equation}
  H_t = L_t H_{t-1} + \widehat U_t^{(r)},
  \qquad
  \widehat U_t^{(r)} = \frac{1}{\sqrt r} B_t X_t^\top,
  \label{eq:exp-no-right-rank}
\end{equation}
where increasing \(r\) increases the rank of the source written at the current
step. SHiPPO-style variants instead use
\begin{equation}
  H_t = L_t H_{t-1} R_t + \widehat U_t,
  \label{eq:exp-shippo-right}
\end{equation}
so that a future operation can act on the channel coordinates of memory that was
written earlier. We report means and sample standard deviations across random
seeds. Appendix~\ref{app:formal-diagnostics} gives the paired-diagnostic task
definitions, model families, metrics, training protocol, numerical summaries,
and the group-local audit; Appendix~\ref{app:transport-mqar-details} gives the
Transport-MQAR generator, controls, length sweeps, and controller
counterfactuals.

\subsection{Paired noncommutative transport diagnostic}
\label{subsec:paired-transport}
\label{subsec:e0-formal}

The primary diagnostic writes a payload vector \(v\) into memory and then applies
two operation tokens. The two paired examples share the same \(v\) and operation
parameters but reverse the order of the operations, \(ab\) versus \(ba\). The
target paired difference is
\begin{equation}
  \Delta_{\mathrm{true}}
  = v^\top (R_a R_b - R_b R_a).
  \label{eq:pair-delta-true}
\end{equation}
We measure the normalized mean-squared error of this paired difference,
\begin{equation}
  \mathrm{Pair}\,\Delta\mathrm{NMSE}
  = \frac{\lVert \Delta_{\mathrm{pred}} - \Delta_{\mathrm{true}} \rVert_2^2}
          {\max\{\lVert \Delta_{\mathrm{true}} \rVert_2^2,\epsilon_{\mathrm{den}}\}},
  \qquad \epsilon_{\mathrm{den}}=10^{-8}.
  \label{eq:pair-delta-nmse}
\end{equation}
Source injection is allowed at WRITE tokens, but operation tokens do not inject
payload-dependent source. Thus a no-right model can increase the rank of what it
writes, but it cannot make a later operation right-multiply previously written
channel coordinates. In contrast, a SHiPPO-style model can apply \(R_t\) to
already-written memory.

\subsection{Right-transport parameterizations and intervention}
\label{subsec:right-transport-parameterizations}
\label{subsec:rdesign}

The oracle-\(R\) variant tests whether the true right action can solve the
paired-transport diagnostic. A stronger test is whether learned and selective
right-transport parameterizations can do so as well. We therefore compare oracle
transport, learned Lie transports initialized from true, zero, or random
generators, and a selective right-transport controller. We also perform an
evaluation-time intervention replacing \(R_t\) by identity while leaving all
other learned weights fixed; this tests whether the paired-difference behavior
is mediated by the right action.

\begin{figure}[H]
\centering
\includegraphics[width=\linewidth]{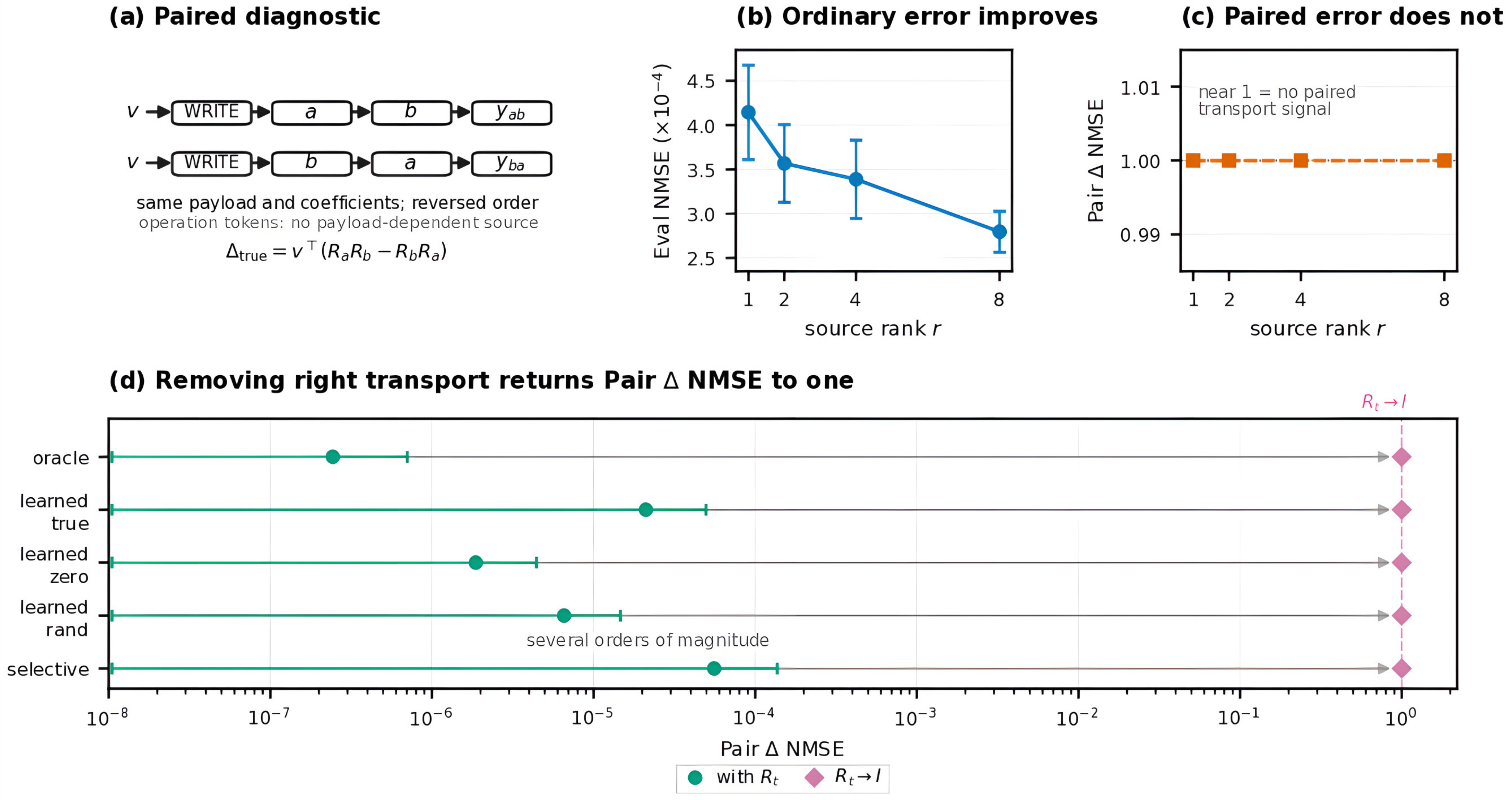}
\caption{Main paired-transport diagnostic. (a) Paired examples share the
same payload and operation coefficients but reverse order; operation tokens
inject no payload-dependent source. (b,c) Increasing no-right source rank
improves ordinary NMSE but leaves Pair \(\Delta\) NMSE near one. (d)
Right-transport variants recover the paired difference, while the evaluation-time
\(R_t\!\to I\) intervention returns Pair \(\Delta\) NMSE to one. Points show
means over seeds; error bars show sample standard deviations.}
\label{fig:main-evidence}
\end{figure}

The key comparison is the intervention in Figure~\ref{fig:main-evidence}(d):
the same trained models lose the paired-difference signal when \(R_t\) is
replaced by identity. This makes the diagnostic more than a capacity
comparison, because the recovered signal is tied to the presence of a future
right action on stored memory. Full numerical summaries are in
Appendix~\ref{app:full-formal-results}.

\subsection{Transport-MQAR as a complementary autoregressive diagnostic}
\label{subsec:transport-mqar-main}

The paired diagnostic isolates the mechanism in a minimal controlled setting. We
next ask whether scan-compatible SHiPPO realizations also help in an
autoregressive finite-field recall task where bindings, operations, and queries
are interleaved. Transport-MQAR inserts invertible operation tokens between
key--value bindings and queries, so the target must be recovered in the current
transported frame. Appendix~\ref{app:transport-mqar-details} gives the
generator, metrics, model geometry, training protocol, full length sweep, and
the controller-suffix counterfactual; the main text reports only the
length-4096 summary needed for the mechanistic comparison.
Transport-MQAR is a finite-field, right-transport modification of multi-query
associative recall (MQAR), a diagnostic introduced to study recall in efficient
language models \citep{arora2024zoology}. The comparison includes standard
GRU \citep{cho2014learning} and Transformer \citep{vaswani2017attention}
baselines.

\begin{table}[t]
\centering
\small
\setlength{\tabcolsep}{5pt}
\caption{Transport-MQAR summary at length 4096. Full length sweeps, standard
 deviations, model geometry, and training protocol are in
 Appendix~\ref{app:transport-mqar-details}.}
\label{tab:transport-mqar-summary}
\begin{tabular}{lcc}
\toprule
Model & Coord. & Exact \\
\midrule
GRU & 0.095 & 0.016 \\
Transformer & 0.042 & 0.002 \\
Free enc/dec & 0.092 & 0.009 \\
No-right & 0.102 & 0.016 \\
DirectGen-SingleExp & 0.103 & \textbf{0.036} \\
StructGen-Split & \textbf{0.110} & 0.033 \\
\midrule
StructGen-Split + suffix zeroing & 0.104 & 0.018 \\
\bottomrule
\end{tabular}
\end{table}

Table~\ref{tab:transport-mqar-summary} summarizes the length-4096
Transport-MQAR results. StructGen-Split modestly improves coordinate accuracy
over no-right and static-basis controls, while DirectGen-SingleExp obtains the
strongest exact accuracy. The suffix-zeroing counterfactual reduces exact
accuracy from 0.033 to 0.018, suggesting that the learned controller
coordinates are used, but not identifying the learned transport geometry.

\subsection{Scope of the evidence}
\label{subsec:experiment-scope}

The experiments support a mechanistic claim in controlled diagnostics:
high-rank current-step source/write updates and future right transport are
different operations on matrix memory. They do not establish broad
sequence-modeling superiority. The formal separation is clearest in the
full-transport paired diagnostic; Transport-MQAR provides complementary
autoregressive evidence but does not identify the learned transport geometry.
The group-local restriction is audited in Appendix~\ref{app:group-audit}.

\section{Discussion}
\label{sec:discussion}

SHiPPO should be viewed as a transported online-memory prior. Once a
right-transport path is realized, the state has ordinary HiPPO coefficients in a
tied moving channel frame, and the right-action term is the gauge term induced
by transporting the approximation family and metric together. This distinguishes
SHiPPO from untied matrix-state recurrences, high-rank source/write updates, or
encoder--decoder factorizations that are not tied by a common transport path.

The scan-compatible cell of Section~\ref{sec:selective-transport-cell} is one
restricted realization, chosen to retain exact block-affine scan and recurrent
decoding. The collapse criterion gives the complementary limitation:
simultaneously reducible right transports collapse to static mixing plus
independent scalar or blockwise banks.

The experiments support this distinction in controlled diagnostics rather than
claiming broad sequence-modeling dominance. Paired diagnostics separate
current-step writes from future right transport, while Transport-MQAR provides
complementary autoregressive evidence and leaves the preferred right-action
realization open.

\section*{Impact Statement}
This paper presents foundational work on recurrent memory priors for sequence
models. The empirical evaluation is primarily synthetic and diagnostic: we do
not release a large pretrained generative model, and we do not claim
deployment-ready language-modeling performance. Potential positive impacts
include better tools for understanding memorization, structured generalization,
and memory mechanisms in efficient sequence models. Potential negative impacts
are indirect and similar to those of improved sequence-modeling methods more
broadly, including possible use in more capable generative systems. We therefore
emphasize diagnostic scope, transparent limitations, and reproducibility.

\section*{Acknowledgements}
\paragraph{Funding.}
This work was supported by JSPS KAKENHI Grant Number JP26K21295.

\paragraph{Computing resources.}
This research used resources of the Argonne Leadership Computing Facility under
ALCF Allocation IDs 15652 and 15654, and resources of the Oak Ridge Leadership
Computing Facility under OLCF Project ID CSC704 through Director's
Discretionary allocation awards.

\bibliographystyle{plainnat}
\bibliography{references}

\newpage
\appendix


\section{Derivations for Section~\ref{sec:shippo}}
\label{app:sec2-derivations}

This appendix supports the operator-level claims of Section~\ref{sec:shippo}.
We first recall the ordinary vector-valued HiPPO variational equations, then derive the SHiPPO normal equation directly from the transported online approximation objective.
We next record the transport-isometry conjugacy, prove Theorem~\ref{thm:shippo-dynamics} and Corollary~\ref{cor:shippo-lift}, and finally explain why transporting only the channel metric generally fails to preserve a finite HiPPO-style coefficient closure.

\subsection{Notation, standing assumptions, and pathwise semantics}
\label{app:notation-assumptions}
\label{app:full-derivations:notation}

Fix a time \(t>0\).
Throughout, signals
\(u(\tau),v(\tau),f(\tau)\in\mathbb R^d\) are treated as column vectors, and the
basis \(\Phi_t(\tau)\in\mathbb R^N\) is also a column vector.
Coefficient matrices are \(C\in\mathbb R^{N\times d}\), so that
\(C^\top\Phi_t(\tau)\in\mathbb R^d\).

We use the Frobenius inner product
\[
\langle A,B\rangle_F:=\mathrm{tr}(A^\top B)
\qquad
(A,B\in\mathbb R^{N\times d}),
\]
so that first variations can be written as
\[
\delta J(C;\Delta)=\langle \nabla_C J(C),\Delta\rangle_F .
\]

\paragraph{Measures and closure.}
For each \(t\), let \(\mu_t\) be a measure on \([0,t]\), and define the Gram
matrix
\[
G(t):=\int_0^t \Phi_t(\tau)\Phi_t(\tau)^\top\,d\mu_t(\tau).
\]
We assume \(G(t)\) is invertible for all \(t\) of interest.
When \(d\mu_t(\tau)=w_t(\tau)d\tau\), define
\[
\psi(t,\tau):=w_t(\tau)\Phi_t(\tau),
\]
and assume the standard HiPPO closure condition
\[
\partial_t\psi(t,\tau)=A_L(t)\psi(t,\tau)\quad(\tau<t),
\qquad
B_L(t):=\psi(t,t).
\]
We also assume sufficient regularity to justify the first variations and
Leibniz-rule differentiations below, for instance the usual dominated
convergence hypotheses.
Under only weak regularity, the displayed coefficient dynamics are interpreted
at times where the relevant derivatives exist, or almost everywhere.

\paragraph{Right-transport paths and pathwise semantics.}
For the SHiPPO construction, let
\(A_R:[0,T]\to\mathbb R^{d\times d}\) be an admissible integrable right-generator path, and let
\(P(t,\tau)\) solve
\[
\partial_tP(t,\tau)=P(t,\tau)A_R(t),
\qquad
P(\tau,\tau)=I_d.
\]
Standard linear ODE theory gives \(P(t,\tau)\in GL(d)\).
The reduction and invertibility details are collected in
Appendix~\ref{app:reductions-moving-frame}.
We use the induced metric
\[
M_P(t,\tau):=P(t,\tau)P(t,\tau)^\top.
\]

SHiPPO is defined relative to such a chosen right-generator path.
For each chosen path, the transport \(P(t,\tau)\), the metric
\(M_P(t,\tau)\), and the transported approximation family are determined, and
the inner online projection problem minimizes only over the coefficient matrix
\(C\).
If \(A_R\) is generated by a causal controller, the variational identities,
normal equations, and coefficient dynamics below apply pathwise for the
realized transport path.
Learning or parameterizing \(A_R\) is an outer modeling problem, not part of
the inner projection optimization.

\subsection{Ordinary vector-valued HiPPO}
\label{app:ordinary-hippo}
\label{app:full-derivations:hippo}

For ordinary HiPPO at time \(t\), the approximation family is
\[
\mathcal G_t^{\mathrm H}
=
\{\,\tau\mapsto C^\top\Phi_t(\tau):
C\in\mathbb R^{N\times d}\,\},
\]
and the Euclidean-channel objective is
\[
J_t^{\mathrm H}(C)
:=
\int_0^t
\bigl\|f(\tau)-C^\top\Phi_t(\tau)\bigr\|_2^2
\,d\mu_t(\tau).
\]

\begin{proposition}[Ordinary HiPPO normal equation]
\label{prop:app-hippo-normal-eq}
Any minimizer \(C_H(t)\) of \(J_t^{\mathrm H}\) satisfies
\[
G(t)C_H(t)
=
\int_0^t
\Phi_t(\tau)f(\tau)^\top\,d\mu_t(\tau).
\]
If \(G(t)\) is invertible, then the minimizer is unique.
\end{proposition}

\begin{proof}
Fix \(t\), and abbreviate \(\Phi(\tau)=\Phi_t(\tau)\).
For a perturbation \(C\mapsto C+\varepsilon\Delta\), with arbitrary
\(\Delta\in\mathbb R^{N\times d}\), define the residual
\[
e(\tau;C):=f(\tau)-C^\top\Phi(\tau).
\]
Then
\[
e(\tau;C+\varepsilon\Delta)
=
e(\tau;C)-\varepsilon\,\Delta^\top\Phi(\tau).
\]
Using
\[
J_t^{\mathrm H}(C)
=
\int_0^t e(\tau;C)^\top e(\tau;C)\,d\mu_t(\tau),
\]
we obtain
\begin{align*}
\delta J_t^{\mathrm H}(C;\Delta)
&=
2\int_0^t
\bigl(\delta e(\tau;C)\bigr)^\top e(\tau;C)
\,d\mu_t(\tau)
\\
&=
-2\int_0^t
\bigl(\Delta^\top\Phi(\tau)\bigr)^\top e(\tau;C)
\,d\mu_t(\tau).
\end{align*}
Since
\[
\bigl(\Delta^\top\Phi\bigr)^\top e
=
\mathrm{tr}\!\left(\Delta^\top\Phi e^\top\right),
\]
we can write
\[
\delta J_t^{\mathrm H}(C;\Delta)
=
-2\,\mathrm{tr}\!\left[
\Delta^\top
\int_0^t
\Phi(\tau)e(\tau;C)^\top\,d\mu_t(\tau)
\right].
\]
Stationarity for all \(\Delta\) is equivalent to
\[
\int_0^t
\Phi(\tau)e(\tau;C)^\top\,d\mu_t(\tau)
=
0.
\]
Expanding
\[
e(\tau;C)^\top
=
f(\tau)^\top-\Phi(\tau)^\top C
\]
gives
\[
\int_0^t
\Phi(\tau)f(\tau)^\top\,d\mu_t(\tau)
-
\left(
\int_0^t
\Phi(\tau)\Phi(\tau)^\top\,d\mu_t(\tau)
\right)C
=
0.
\]
Thus
\[
G(t)C
=
\int_0^t
\Phi_t(\tau)f(\tau)^\top\,d\mu_t(\tau).
\]
Uniqueness follows because \(J_t^{\mathrm H}\) is a strictly convex quadratic in
\(C\) whenever \(G(t)\) is invertible.
\end{proof}

\begin{proposition}[Ordinary HiPPO coefficient dynamics under closure]
\label{prop:app-hippo-dynamics}
Assume \(d\mu_t(\tau)=w_t(\tau)d\tau\) and the closure condition above. If
\(G(t)=I_N\), then
\[
C_H(t)
=
\int_0^t
\psi(t,\tau)f(\tau)^\top\,d\tau,
\]
and
\[
\dot C_H(t)
=
A_L(t)C_H(t)+B_L(t)f(t)^\top.
\]
\end{proposition}

\begin{proof}
With \(G(t)=I_N\), Proposition~\ref{prop:app-hippo-normal-eq} gives
\[
C_H(t)
=
\int_0^t
\Phi_t(\tau)f(\tau)^\top\,d\mu_t(\tau)
=
\int_0^t
\psi(t,\tau)f(\tau)^\top\,d\tau.
\]
Differentiate using Leibniz' rule:
\[
\dot C_H(t)
=
\psi(t,t)f(t)^\top
+
\int_0^t
\partial_t\psi(t,\tau)f(\tau)^\top\,d\tau.
\]
Substituting
\[
\partial_t\psi(t,\tau)=A_L(t)\psi(t,\tau)
\quad(\tau<t),
\qquad
B_L(t)=\psi(t,t),
\]
yields
\[
\dot C_H(t)
=
B_L(t)f(t)^\top
+
A_L(t)
\int_0^t
\psi(t,\tau)f(\tau)^\top\,d\tau
=
A_L(t)C_H(t)+B_L(t)f(t)^\top.
\]
\end{proof}

\paragraph{Non-orthonormal bases.}
If \(G(t)\neq I_N\), then
\[
C_H(t)=G(t)^{-1}b(t),
\qquad
b(t):=\int_0^t
\Phi_t(\tau)f(\tau)^\top\,d\mu_t(\tau).
\]
Differentiating gives
\[
\dot C_H(t)
=
G(t)^{-1}\dot b(t)
-
G(t)^{-1}\dot G(t)C_H(t).
\]
The main text focuses on the orthonormalized case \(G(t)=I_N\), where the
coefficient dynamics take the clean HiPPO form.

\subsection{Direct stationarity of the SHiPPO online approximation problem}
\label{app:shippo-stationarity}

For fixed \(t\), the SHiPPO approximation family is
\[
\mathcal G_t^{\mathrm{SH}}
=
\{\,\tau\mapsto P(t,\tau)^{-\top}C^\top\Phi_t(\tau):
C\in\mathbb R^{N\times d}\,\}.
\]
Equivalently, the SHiPPO coefficient matrix is the minimizer of
\[
J_t^{\mathrm{SH}}(C)
:=
\int_0^t
e_S(\tau;C)^\top M_P(t,\tau)e_S(\tau;C)
\,d\mu_t(\tau),
\]
where
\[
e_S(\tau;C)
:=
f(\tau)-P(t,\tau)^{-\top}C^\top\Phi_t(\tau).
\]

\begin{proposition}[SHiPPO normal equation from stationarity]
\label{prop:app-shippo-stationarity}
Any minimizer \(C_S(t)\) of \(J_t^{\mathrm{SH}}\) satisfies
\[
G(t)C_S(t)
=
\int_0^t
\Phi_t(\tau)f(\tau)^\top P(t,\tau)\,d\mu_t(\tau).
\]
If \(G(t)\) is invertible, then the minimizer is unique.
\end{proposition}

\begin{proof}
Fix \(t\), abbreviate
\[
P(\tau):=P(t,\tau),
\qquad
M_P(\tau):=M_P(t,\tau),
\qquad
\Phi(\tau):=\Phi_t(\tau).
\]
For a perturbation \(C\mapsto C+\varepsilon\Delta\),
\[
\delta e_S(\tau;C)
=
-P(\tau)^{-\top}\Delta^\top\Phi(\tau).
\]
Therefore
\begin{align*}
\delta J_t^{\mathrm{SH}}(C;\Delta)
&=
2\int_0^t
\bigl(\delta e_S(\tau;C)\bigr)^\top
M_P(\tau)e_S(\tau;C)
\,d\mu_t(\tau)
\\
&=
-2\int_0^t
\bigl(\Delta^\top\Phi(\tau)\bigr)^\top
P(\tau)^{-1}M_P(\tau)e_S(\tau;C)
\,d\mu_t(\tau).
\end{align*}
Since
\[
P(\tau)^{-1}M_P(\tau)
=
P(\tau)^{-1}P(\tau)P(\tau)^\top
=
P(\tau)^\top,
\]
we get
\[
\delta J_t^{\mathrm{SH}}(C;\Delta)
=
-2\int_0^t
\bigl(\Delta^\top\Phi(\tau)\bigr)^\top
P(\tau)^\top e_S(\tau;C)
\,d\mu_t(\tau).
\]
Using the trace identity
\[
\bigl(\Delta^\top\Phi\bigr)^\top P^\top e_S
=
\mathrm{tr}\!\left(\Delta^\top\Phi e_S^\top P\right),
\]
this becomes
\[
\delta J_t^{\mathrm{SH}}(C;\Delta)
=
-2\,\mathrm{tr}\!\left[
\Delta^\top
\int_0^t
\Phi(\tau)e_S(\tau;C)^\top P(\tau)
\,d\mu_t(\tau)
\right].
\]
Stationarity for all \(\Delta\) is equivalent to
\[
\int_0^t
\Phi(\tau)e_S(\tau;C)^\top P(\tau)
\,d\mu_t(\tau)
=
0.
\]
Now
\[
e_S(\tau;C)^\top P(\tau)
=
f(\tau)^\top P(\tau)
-
\Phi(\tau)^\top C,
\]
because
\[
\bigl(P(\tau)^{-\top}C^\top\Phi(\tau)\bigr)^\top P(\tau)
=
\Phi(\tau)^\top C P(\tau)^{-1}P(\tau)
=
\Phi(\tau)^\top C.
\]
Thus stationarity gives
\[
\int_0^t
\Phi(\tau)f(\tau)^\top P(\tau)\,d\mu_t(\tau)
-
\left(
\int_0^t
\Phi(\tau)\Phi(\tau)^\top\,d\mu_t(\tau)
\right)C
=
0,
\]
or equivalently
\[
G(t)C
=
\int_0^t
\Phi_t(\tau)f(\tau)^\top P(t,\tau)\,d\mu_t(\tau).
\]
This proves the displayed normal equation for \(C=C_S(t)\).

Uniqueness follows from strict convexity.
Indeed, the map
\(C\mapsto P(t,\tau)^{-\top}C^\top\Phi_t(\tau)\) is injective modulo
\(\mu_t\)-null sets when \(G(t)\) is invertible and \(P(t,\tau)\in GL(d)\), and
the metric \(M_P(t,\tau)\) is positive definite for each \(\tau\).
\end{proof}

\subsection{Transport isometry and SHiPPO--HiPPO conjugacy}
\label{app:shippo-conjugacy}
\label{app:full-derivations:conjugacy}

The direct stationarity calculation above gives the SHiPPO normal equation.
We now record the equivalent conjugacy view, which explains why the coupled
transport of the metric and approximation family preserves the HiPPO coefficient
structure.

For fixed \(t\), define the transport operator
\[
(\mathcal T_tu)(\tau):=P(t,\tau)^\top u(\tau),
\qquad
(\mathcal T_t^{-1}u)(\tau):=P(t,\tau)^{-\top}u(\tau).
\]
Also define the metric inner product
\[
\langle u,v\rangle_{t,M_P}
:=
\int_0^t
u(\tau)^\top M_P(t,\tau)v(\tau)\,d\mu_t(\tau),
\]
and the Euclidean-channel inner product
\[
\langle u,v\rangle_{t,I}
:=
\int_0^t
u(\tau)^\top v(\tau)\,d\mu_t(\tau).
\]

\begin{proposition}[Transport isometry and SHiPPO--HiPPO conjugacy]
\label{prop:app-transport-conjugacy}
For any \(u,v:[0,t]\to\mathbb R^d\),
\[
\langle u,v\rangle_{t,M_P}
=
\langle \mathcal T_tu,\mathcal T_tv\rangle_{t,I}.
\]
Moreover,
\[
\mathcal T_t
\Bigl(
P(t,\cdot)^{-\top}C^\top\Phi_t(\cdot)
\Bigr)
=
C^\top\Phi_t(\cdot),
\]
so \(\mathcal T_t\) maps
\(\mathcal G_t^{\mathrm{SH}}\) bijectively onto
\(\mathcal G_t^{\mathrm H}\).
Consequently,
\[
\operatorname{shippo}_t(f)
=
\operatorname{hippo}_t(\mathcal T_tf).
\]
\end{proposition}

\begin{proof}
For the isometry, compute pointwise:
\[
u(\tau)^\top M_P(t,\tau)v(\tau)
=
u(\tau)^\top P(t,\tau)P(t,\tau)^\top v(\tau)
=
(\mathcal T_tu)(\tau)^\top(\mathcal T_tv)(\tau).
\]
Integrating over \([0,t]\) gives
\[
\langle u,v\rangle_{t,M_P}
=
\langle \mathcal T_tu,\mathcal T_tv\rangle_{t,I}.
\]

For the mapping of approximation families,
\[
\mathcal T_t
\left(
P(t,\tau)^{-\top}C^\top\Phi_t(\tau)
\right)
=
P(t,\tau)^\top P(t,\tau)^{-\top}C^\top\Phi_t(\tau)
=
C^\top\Phi_t(\tau).
\]
Thus
\[
\mathcal T_t(\mathcal G_t^{\mathrm{SH}})
=
\mathcal G_t^{\mathrm H}.
\]
Since \(P(t,\tau)\in GL(d)\), the inverse map is
\(\mathcal T_t^{-1}\), so this correspondence is bijective.

Finally, for any \(g\in\mathcal G_t^{\mathrm{SH}}\),
\[
\|f-g\|_{t,M_P}^2
=
\|\mathcal T_tf-\mathcal T_tg\|_{t,I}^2.
\]
Minimizing over \(g\in\mathcal G_t^{\mathrm{SH}}\) is therefore equivalent to
minimizing over
\[
\tilde g:=\mathcal T_tg\in\mathcal G_t^{\mathrm H}.
\]
This is precisely the ordinary HiPPO projection of \(\mathcal T_tf\).
The coefficient matrix is unchanged under the map between the two approximation
families, giving
\[
\operatorname{shippo}_t(f)
=
\operatorname{hippo}_t(\mathcal T_tf).
\]
\end{proof}

\paragraph{Relation to the normal equation.}
Proposition~\ref{prop:app-transport-conjugacy} gives another way to obtain
Proposition~\ref{prop:app-shippo-stationarity}: apply the ordinary HiPPO
normal equation to the transformed signal
\[
(\mathcal T_tf)(\tau)=P(t,\tau)^\top f(\tau).
\]
More importantly, the conjugacy shows that SHiPPO is ordinary HiPPO in a
transported channel frame.
This is the operator-level reason that
Definition~\ref{def:shippo-projection} transports the metric and the
approximation family together, rather than modifying the channel metric alone.

\subsection{Proof of Theorem~\ref{thm:shippo-dynamics} and Corollary~\ref{cor:shippo-lift}}
\label{app:shippo-theorem}
\label{app:full-derivations:shippo-dynamics}

\begin{proof}
By Proposition~\ref{prop:app-shippo-stationarity}, the SHiPPO minimizer
satisfies
\[
G(t)C_S(t)
=
\int_0^t
\Phi_t(\tau)f(\tau)^\top P(t,\tau)\,d\mu_t(\tau).
\]
This is the first displayed equation in Theorem~\ref{thm:shippo-dynamics}.

Now assume \(G(t)=I_N\) and \(d\mu_t(\tau)=w_t(\tau)d\tau\), so that
\(\psi(t,\tau)=w_t(\tau)\Phi_t(\tau)\).
Then the normal equation becomes
\[
C_S(t)
=
\int_0^t
\Phi_t(\tau)f(\tau)^\top P(t,\tau)\,d\mu_t(\tau)
=
\int_0^t
\psi(t,\tau)f(\tau)^\top P(t,\tau)\,d\tau.
\]
Differentiating this representation using Leibniz' rule gives
\begin{align*}
\dot C_S(t)
&=
\psi(t,t)f(t)^\top P(t,t)
\\
&\quad
+
\int_0^t
\partial_t\psi(t,\tau)f(\tau)^\top P(t,\tau)\,d\tau
\\
&\quad
+
\int_0^t
\psi(t,\tau)f(\tau)^\top \partial_tP(t,\tau)\,d\tau.
\end{align*}
Using
\[
P(t,t)=I_d,
\qquad
\psi(t,t)=B_L(t),
\]
the boundary term is
\[
B_L(t)f(t)^\top.
\]
Using the closure condition
\[
\partial_t\psi(t,\tau)=A_L(t)\psi(t,\tau),
\]
the second term is
\[
A_L(t)
\int_0^t
\psi(t,\tau)f(\tau)^\top P(t,\tau)\,d\tau
=
A_L(t)C_S(t).
\]
Using the transport equation
\[
\partial_tP(t,\tau)=P(t,\tau)A_R(t),
\]
the third term is
\[
\int_0^t
\psi(t,\tau)f(\tau)^\top P(t,\tau)A_R(t)\,d\tau
=
C_S(t)A_R(t),
\]
since \(A_R(t)\) does not depend on \(\tau\) inside the integral.
Therefore
\[
\dot C_S(t)
=
A_L(t)C_S(t)+B_L(t)f(t)^\top+C_S(t)A_R(t),
\]
which proves the Sylvester coefficient dynamics.
\end{proof}

\paragraph{Proof of Corollary~\ref{cor:shippo-lift}.}
The corollary is not a statement about arbitrary matrix ODEs. It applies to
closed one-sided coefficient equations obtained from the online projection setup
above. Suppose that, under the HiPPO closure condition and the orthonormalized
normal equation, the ordinary coefficient trajectory satisfies
\[
\dot C(t)=A_L(t)C(t)+B_L(t)f(t)^\top .
\]
For any admissible right-transport path \(A_R\), Definition~\ref{def:shippo-projection}
uses the same left approximation family and source injection, while replacing
the signal by its transported pathwise history in the normal equation. Hence
\[
C_S(t)
=
\int_0^t \psi(t,\tau)f(\tau)^\top P(t,\tau)\,d\tau .
\]
Differentiating this representation at differentiability times gives the same
left-memory contribution and boundary source as in the one-sided equation,
while the right transport contributes
\[
\int_0^t
\psi(t,\tau)f(\tau)^\top P(t,\tau)A_R(t)\,d\tau
=
C_S(t)A_R(t).
\]
Therefore
\[
\dot C_S(t)
=
A_L(t)C_S(t)+B_L(t)f(t)^\top+C_S(t)A_R(t).
\]
Thus the transported lift preserves the left online-memory operator inherited
from the projection problem and adds the chosen right-action transport as an
external modeling path. Architectural restrictions on \(A_R\), such as
group-local or controller-compatible forms, enter only in the scan-compatible
realization of Section~\ref{sec:selective-transport-cell}.

\paragraph{Non-orthonormal bases.}
If \(G(t)\neq I_N\), define
\[
b_S(t)
:=
\int_0^t
\Phi_t(\tau)f(\tau)^\top P(t,\tau)\,d\mu_t(\tau).
\]
Then
\[
C_S(t)=G(t)^{-1}b_S(t),
\]
and differentiating gives
\[
\dot C_S(t)
=
G(t)^{-1}\dot b_S(t)
-
G(t)^{-1}\dot G(t)C_S(t).
\]
Thus the clean Sylvester form in the main text is the orthonormalized
\(G(t)=I_N\) case, matching the usual HiPPO presentation.

\subsection{Metric-only modifications do not generally close}
\label{app:metric-only}
\label{app:full-derivations:metric-only}

This subsection justifies the coupled transport used in
Definition~\ref{def:shippo-projection} by analyzing a nearby but different
construction in which only the channel metric is changed.
This construction is not part of SHiPPO itself and, in general, does not
preserve a finite HiPPO-style coefficient closure.

Keep the ordinary HiPPO family
\[
\mathcal G_t^{\mathrm H}
=
\{\,\tau\mapsto C^\top\Phi_t(\tau):
C\in\mathbb R^{N\times d}\,\},
\]
and replace the Euclidean channel metric by a possibly \(t\)- and
\(\tau\)-dependent SPD matrix
\[
M(t,\tau)\in\mathbb S_{++}^d.
\]
Define the metric-only objective
\[
J_t^{M}(C)
:=
\int_0^t
\bigl(f(\tau)-C^\top\Phi_t(\tau)\bigr)^\top
M(t,\tau)
\bigl(f(\tau)-C^\top\Phi_t(\tau)\bigr)
\,d\mu_t(\tau).
\]

\begin{proposition}[Stationary condition for the metric-only extension]
\label{prop:app-metric-only-stationary}
Any minimizer \(C_M(t)\) of \(J_t^{M}\) satisfies
\[
\int_0^t
\Phi_t(\tau)
\left(
M(t,\tau)
\bigl(f(\tau)-C_M(t)^\top\Phi_t(\tau)\bigr)
\right)^\top
\,d\mu_t(\tau)
=
0.
\]
Equivalently,
\[
\int_0^t
\Phi_t(\tau)f(\tau)^\top M(t,\tau)\,d\mu_t(\tau)
=
\int_0^t
\Phi_t(\tau)\Phi_t(\tau)^\top C_M(t)M(t,\tau)\,d\mu_t(\tau).
\]
\end{proposition}

\begin{proof}
Fix \(t\), abbreviate
\[
\Phi(\tau)=\Phi_t(\tau),
\qquad
M(\tau)=M(t,\tau),
\]
and define
\[
e(\tau;C):=f(\tau)-C^\top\Phi(\tau).
\]
For \(C\mapsto C+\varepsilon\Delta\),
\[
e(\tau;C+\varepsilon\Delta)
=
e(\tau;C)-\varepsilon\,\Delta^\top\Phi(\tau).
\]
Since \(M(\tau)\) is symmetric,
\begin{align*}
\delta J_t^M(C;\Delta)
&=
2\int_0^t
\bigl(\delta e(\tau;C)\bigr)^\top M(\tau)e(\tau;C)
\,d\mu_t(\tau)
\\
&=
-2\int_0^t
\bigl(\Delta^\top\Phi(\tau)\bigr)^\top M(\tau)e(\tau;C)
\,d\mu_t(\tau).
\end{align*}
Using
\[
\bigl(\Delta^\top\Phi\bigr)^\top M e
=
\mathrm{tr}\!\left(\Delta^\top\Phi(Me)^\top\right),
\]
we get
\[
\delta J_t^M(C;\Delta)
=
-2\,\mathrm{tr}\!\left[
\Delta^\top
\int_0^t
\Phi(\tau)\bigl(M(\tau)e(\tau;C)\bigr)^\top
\,d\mu_t(\tau)
\right].
\]
Stationarity for all \(\Delta\) yields
\[
\int_0^t
\Phi(\tau)\bigl(M(\tau)e(\tau;C)\bigr)^\top
\,d\mu_t(\tau)
=
0.
\]
Expanding \(e(\tau;C)=f(\tau)-C^\top\Phi(\tau)\) gives
\[
\int_0^t
\Phi(\tau)f(\tau)^\top M(\tau)\,d\mu_t(\tau)
=
\int_0^t
\Phi(\tau)\Phi(\tau)^\top C M(\tau)\,d\mu_t(\tau).
\]
Substituting \(C=C_M(t)\) proves the claim.
\end{proof}

It is convenient to write this stationary equation as a linear operator equation.
Define
\[
\mathcal K_{M,t}[C]
:=
\int_0^t
\Phi_t(\tau)\Phi_t(\tau)^\top C M(t,\tau)
\,d\mu_t(\tau),
\]
and
\[
b_{M,t}
:=
\int_0^t
\Phi_t(\tau)f(\tau)^\top M(t,\tau)
\,d\mu_t(\tau).
\]
Then
\[
\mathcal K_{M,t}[C_M(t)]
=
b_{M,t}.
\]

\begin{proposition}[\(\tau\)-independent metrics reduce to ordinary HiPPO]
\label{prop:app-tau-indep-reduction}
If
\[
M(t,\tau)=\overline M(t)\in\mathbb S_{++}^d
\]
is independent of \(\tau\), then the metric-only stationary equation reduces to
\[
G(t)C_M(t)
=
\int_0^t
\Phi_t(\tau)f(\tau)^\top\,d\mu_t(\tau),
\]
the same normal equation as ordinary vector-valued HiPPO.
\end{proposition}

\begin{proof}
When \(M(t,\tau)=\overline M(t)\), the operator equation becomes
\[
\left(
\int_0^t
\Phi_t(\tau)\Phi_t(\tau)^\top\,d\mu_t(\tau)
\right)
C_M(t)\overline M(t)
=
\int_0^t
\Phi_t(\tau)f(\tau)^\top\overline M(t)\,d\mu_t(\tau).
\]
Since \(\overline M(t)\) is invertible, right-multiplying by
\(\overline M(t)^{-1}\) gives
\[
G(t)C_M(t)
=
\int_0^t
\Phi_t(\tau)f(\tau)^\top\,d\mu_t(\tau).
\]
\end{proof}

\paragraph{Why \(\tau\)-dependent metrics are obstructive.}
When \(M(t,\tau)\) genuinely depends on \(\tau\), the operator
\(\mathcal K_{M,t}\) does not, in general, factor as
\[
C\mapsto G(t)C\overline M(t)
\]
for some matrix \(\overline M(t)\).
Consequently, \(C_M(t)\) is not generally given by a simple HiPPO-style normal equation.
Recovering an online closed ODE for \(C_M(t)\) typically requires either solving a time-varying linear
operator equation at each \(t\), or augmenting the state with additional
metric-dependent moments.

SHiPPO avoids this generic obstruction by coupling the metric with the
approximation family.
In the transported construction,
\[
M_P(t,\tau)=P(t,\tau)P(t,\tau)^\top
\]
is paired with the transported family
\[
\tau\mapsto P(t,\tau)^{-\top}C^\top\Phi_t(\tau).
\]
This coupling is exactly what produces the stationarity identity
\[
G(t)C_S(t)
=
\int_0^t
\Phi_t(\tau)f(\tau)^\top P(t,\tau)\,d\mu_t(\tau),
\]
and hence the Sylvester coefficient dynamics of
Theorem~\ref{thm:shippo-dynamics}.


\section{Reductions, Identity Metrics, and Moving Frames}
\label{app:reductions-moving-frame}

This appendix supports the reduction, identity-metric, and moving-frame claims used at the end of Section~\ref{sec:shippo}.
Fix an admissible right-generator path \(A_R\in L^1([0,T];\mathbb R^{d\times d})\).
The results below are exact for such a chosen path, and apply pathwise when \(A_R\) is generated by a causal controller.
This appendix is not about optimizing over \(A_R\), but about the algebraic consequences of the transported approximation problem once a transport path has been specified.
Since \(A_R\) is only assumed integrable, differential identities involving \(A_R(t)\) hold almost everywhere; statements such as \(A_R\equiv0\) or \(A_R+A_R^\top\equiv0\) are interpreted in this a.e.\ sense unless additional continuity is assumed.

\subsection{Transport basics}
\label{app:transport-basics}
\label{app:reductions-moving-frame:transport-basics}

Throughout this appendix, assume
\[
A_R:[0,T]\to\mathbb R^{d\times d}
\]
is measurable and integrable. For each fixed \(\tau\in[0,T]\), define
\(P(\cdot,\tau)\) as the unique absolutely continuous solution of
\begin{equation}
P(t,\tau)
=
I_d+\int_\tau^t P(s,\tau)A_R(s)\,ds,
\qquad
0\le \tau\le t\le T.
\label{eq:appB-transport-integral}
\end{equation}
Equivalently,
\begin{equation}
\partial_tP(t,\tau)=P(t,\tau)A_R(t)
\quad\text{for a.e. }t\in[\tau,T],
\qquad
P(\tau,\tau)=I_d.
\label{eq:appB-transport-ode}
\end{equation}

\begin{lemma}[Invertibility and inverse dynamics]
\label{lem:transport-invertibility}
\label{lem:appB-invertibility}
For every \(0\le \tau\le t\le T\), \(P(t,\tau)\in GL(d)\). Moreover,
\[
\partial_tP(t,\tau)^{-1}
=
-A_R(t)P(t,\tau)^{-1}
\quad\text{for a.e. }t\in[\tau,T],
\qquad
P(\tau,\tau)^{-1}=I_d.
\]
The determinant satisfies Liouville's formula
\[
\det P(t,\tau)
=
\exp\!\left(\int_\tau^t \mathrm{tr}\,A_R(s)\,ds\right)
\neq 0.
\]
\end{lemma}

\begin{proof}
Define \(Y(\cdot,\tau)\) as the absolutely continuous solution of
\[
\partial_tY(t,\tau)
=
-A_R(t)Y(t,\tau)
\quad\text{for a.e. }t\in[\tau,T],
\qquad
Y(\tau,\tau)=I_d.
\]
Then, using \(\partial_tP=P A_R\),
\[
\partial_t\bigl(P(t,\tau)Y(t,\tau)\bigr)
=
P(t,\tau)A_R(t)Y(t,\tau)
-
P(t,\tau)A_R(t)Y(t,\tau)
=
0
\]
for a.e. \(t\). Since \(P(\tau,\tau)Y(\tau,\tau)=I_d\), we get
\[
P(t,\tau)Y(t,\tau)=I_d
\]
for all \(t\ge \tau\). Because the matrices are square, the existence of this
right inverse implies \(P(t,\tau)\in GL(d)\), and hence
\[
Y(t,\tau)=P(t,\tau)^{-1}.
\]
This proves the inverse dynamics.

For the determinant, since \(P(t,\tau)\) is absolutely continuous and invertible,
Jacobi's formula gives, for a.e. \(t\),
\[
\frac{d}{dt}\log\det P(t,\tau)
=
\mathrm{tr}\!\left(P(t,\tau)^{-1}\partial_tP(t,\tau)\right)
=
\mathrm{tr}\!\left(P(t,\tau)^{-1}P(t,\tau)A_R(t)\right)
=
\mathrm{tr}\,A_R(t).
\]
Integrating from \(\tau\) to \(t\), and using \(\det P(\tau,\tau)=1\), yields
\[
\det P(t,\tau)
=
\exp\!\left(\int_\tau^t \mathrm{tr}\,A_R(s)\,ds\right).
\]
The right-hand side is nonzero, completing the proof.
\end{proof}

\begin{lemma}[Composition rule]
\label{lem:appB-composition-rule}
For \(0\le \tau\le \sigma\le t\le T\),
\[
P(t,\tau)
=
P(\sigma,\tau)P(t,\sigma).
\]
\end{lemma}

\begin{proof}
Fix \(0\le\tau\le\sigma\le T\), and define
\[
Q(t,\tau):=P(\sigma,\tau)P(t,\sigma),
\qquad t\in[\sigma,T].
\]
Then
\[
\partial_tQ(t,\tau)
=
P(\sigma,\tau)P(t,\sigma)A_R(t)
=
Q(t,\tau)A_R(t)
\]
for a.e. \(t\in[\sigma,T]\), and
\[
Q(\sigma,\tau)
=
P(\sigma,\tau)P(\sigma,\sigma)
=
P(\sigma,\tau).
\]
On the other hand, \(P(t,\tau)\) restricted to \(t\in[\sigma,T]\) satisfies the
same differential equation and the same initial value at \(t=\sigma\). By
uniqueness,
\[
P(t,\tau)=P(\sigma,\tau)P(t,\sigma).
\]
\end{proof}

\paragraph{Order convention.}
The composition rule reflects our right-action convention. If a state evolves as
\(H(t)=H(\tau)P(t,\tau)\), then the segment from \(\tau\) to \(t\) factors as
first applying \(P(\sigma,\tau)\), then applying \(P(t,\sigma)\):
\[
H(t)
=
H(\tau)P(\sigma,\tau)P(t,\sigma).
\]

\subsection{Exact reduction to ordinary HiPPO}
\label{app:exact-reduction}
\label{app:reductions-moving-frame:reduction}

\begin{proposition}[Exact reduction to ordinary HiPPO]
\label{prop:appB-exact-reduction}
The following are equivalent:
\begin{enumerate}
\item \(A_R(t)=0\) for a.e. \(t\in[0,T]\);
\item \(P(t,\tau)=I_d\) for all \(0\le \tau\le t\le T\).
\end{enumerate}
In this case,
\[
M_P(t,\tau)=P(t,\tau)P(t,\tau)^\top=I_d,
\]
the transported approximation family equals the ordinary HiPPO approximation
family, and SHiPPO reduces exactly to ordinary vector-valued HiPPO.
\end{proposition}

\begin{proof}
If \(A_R(t)=0\) for a.e. \(t\), then the integral equation
\eqref{eq:appB-transport-integral} gives
\[
P(t,\tau)
=
I_d+\int_\tau^t P(s,\tau)A_R(s)\,ds
=
I_d
\]
for all \(t\ge\tau\).

Conversely, suppose \(P(t,\tau)=I_d\) for all \(0\le\tau\le t\le T\). Taking
\(\tau=0\), we have \(P(t,0)=I_d\) for all \(t\in[0,T]\). Since
\(\partial_tP(t,0)=P(t,0)A_R(t)\) for a.e. \(t\), we obtain
\[
0=\partial_tP(t,0)=A_R(t)
\]
for a.e. \(t\in[0,T]\).

If \(P(t,\tau)=I_d\), then \(M_P(t,\tau)=I_d\) and
\[
P(t,\tau)^{-\top}C^\top\Phi_t(\tau)
=
C^\top\Phi_t(\tau).
\]
Thus the SHiPPO metric and approximation family are exactly the ordinary
Euclidean-channel HiPPO metric and approximation family.
\end{proof}

\subsection{Identity metric versus identity transport}
\label{app:identity-metric}
\label{app:reductions-moving-frame:identity-metric}

Recall the transport-induced metric
\[
M_P(t,\tau):=P(t,\tau)P(t,\tau)^\top.
\]
The condition \(M_P(t,\tau)=I_d\) does not imply \(P(t,\tau)=I_d\). It only
constrains the symmetric part of the right-transport generator.

\begin{proposition}[Identity metric iff skew transport generator]
\label{prop:appB-metric-identity}
The following are equivalent:
\begin{enumerate}
\item \(M_P(t,\tau)=I_d\) for all \(0\le\tau\le t\le T\);
\item \(A_R(t)+A_R(t)^\top=0\) for a.e. \(t\in[0,T]\).
\end{enumerate}
Consequently, \(M_P\equiv I_d\) does not imply \(P\equiv I_d\) unless the
skew transport is also trivial.
\end{proposition}

\begin{proof}
Fix \(\tau\). Since \(P(\cdot,\tau)\) is absolutely continuous,
\(M_P(\cdot,\tau)\) is absolutely continuous, and for a.e. \(t\ge\tau\),
\begin{align*}
\partial_tM_P(t,\tau)
&=
(\partial_tP(t,\tau))P(t,\tau)^\top
+
P(t,\tau)(\partial_tP(t,\tau))^\top
\\
&=
P(t,\tau)A_R(t)P(t,\tau)^\top
+
P(t,\tau)A_R(t)^\top P(t,\tau)^\top
\\
&=
P(t,\tau)\bigl(A_R(t)+A_R(t)^\top\bigr)P(t,\tau)^\top.
\end{align*}

If \(A_R(t)+A_R(t)^\top=0\) for a.e. \(t\), then
\[
\partial_tM_P(t,\tau)=0
\]
for a.e. \(t\ge\tau\). Since \(M_P(\tau,\tau)=I_d\), it follows that
\[
M_P(t,\tau)=I_d
\]
for all \(t\ge\tau\).

Conversely, suppose \(M_P(t,\tau)=I_d\) for all \(0\le\tau\le t\le T\). Taking
\(\tau=0\), we have \(M_P(t,0)=I_d\) for all \(t\). Hence
\[
0
=
\partial_tM_P(t,0)
=
P(t,0)\bigl(A_R(t)+A_R(t)^\top\bigr)P(t,0)^\top
\]
for a.e. \(t\). By Lemma~\ref{lem:transport-invertibility}, \(P(t,0)\) is
invertible, so
\[
A_R(t)+A_R(t)^\top=0
\]
for a.e. \(t\).
\end{proof}

\paragraph{Implication for the projection problem.}
Even when \(M_P\equiv I_d\), SHiPPO need not reduce to ordinary HiPPO because
the approximation family is still transported:
\[
\mathcal G_t^{\mathrm{SH}}
=
\{\,\tau\mapsto P(t,\tau)^{-\top}C^\top\Phi_t(\tau):
C\in\mathbb R^{N\times d}\,\}.
\]
Equivalently, by Proposition~\ref{prop:app-transport-conjugacy},
\[
\operatorname{shippo}_t(f)
=
\operatorname{hippo}_t(\mathcal T_tf),
\qquad
(\mathcal T_tf)(\tau)=P(t,\tau)^\top f(\tau).
\]
Thus, under \(M_P\equiv I_d\), SHiPPO is ordinary HiPPO applied to a moving-frame
history, not ordinary HiPPO applied to the original history.

\begin{example}[Orthogonal transport with nontrivial moving frame]
\label{ex:appB-rotation}
Let \(d=2\) and set
\[
A_R(t)
=
\omega
\begin{pmatrix}
0 & -1\\
1 & 0
\end{pmatrix},
\qquad
\omega\neq 0.
\]
Then
\[
A_R(t)^\top=-A_R(t),
\]
so Proposition~\ref{prop:appB-metric-identity} gives
\[
M_P(t,\tau)=I_2.
\]
However,
\[
P(t,\tau)
=
\exp\!\bigl((t-\tau)A_R\bigr)
=
\begin{pmatrix}
\cos(\omega(t-\tau)) & -\sin(\omega(t-\tau))\\
\sin(\omega(t-\tau)) & \cos(\omega(t-\tau))
\end{pmatrix},
\]
which is not equal to \(I_2\) for generic \(t\neq\tau\). Therefore SHiPPO does
not reduce to ordinary HiPPO. The right-action term
\[
C_S(t)A_R(t)
\]
in Theorem~\ref{thm:shippo-dynamics} remains present in the coefficient
 dynamics and is generically nonzero.
\end{example}

\subsection{Moving-frame factorization for a chosen transport path}
\label{app:moving-frame}
\label{app:reductions-moving-frame:moving-frame}

This subsection records the exact moving-frame factorization associated with a chosen admissible transport path \(A_R\).
When the transport is controller-generated, the same factorization applies pathwise to the realized trajectory.
Because the frame then depends on the input trajectory and model parameters, the factorization should not be interpreted as an input-independent reduction of SHiPPO to ordinary HiPPO.

\begin{definition}[Fundamental matrix and moving channel frame]
\label{def:appB-global-frame}
Let \(V:[0,T]\to GL(d)\) be the fundamental matrix solving
\[
\dot V(t)=V(t)A_R(t)
\quad\text{for a.e. }t\in[0,T],
\qquad
V(0)=I_d.
\]
Define the history encoder and coefficient decoder by
\[
(\mathcal E_V f)(\tau):=V(\tau)^{-T}f(\tau),
\qquad
\mathcal D_{V,t}(C):=C\,V(t).
\]
\end{definition}

\begin{lemma}[Frame representation of the transport]
\label{lem:appB-frame-repr}
For all \(0\le\tau\le t\le T\),
\[
P(t,\tau)=V(\tau)^{-1}V(t),
\qquad
P(t,\tau)^\top=V(t)^\top V(\tau)^{-T}.
\]
\end{lemma}

\begin{proof}
Fix \(\tau\), and define
\[
Q(t,\tau):=V(\tau)^{-1}V(t),
\qquad t\in[\tau,T].
\]
Then
\[
\partial_tQ(t,\tau)
=
V(\tau)^{-1}\dot V(t)
=
V(\tau)^{-1}V(t)A_R(t)
=
Q(t,\tau)A_R(t)
\]
for a.e. \(t\ge\tau\), and
\[
Q(\tau,\tau)=I_d.
\]
By uniqueness of solutions to \eqref{eq:appB-transport-ode},
\[
Q(t,\tau)=P(t,\tau).
\]
Taking transposes gives
\[
P(t,\tau)^\top=V(t)^\top V(\tau)^{-T}.
\]
\end{proof}

\begin{lemma}[HiPPO equivariance to fixed channel mixing]
\label{lem:appB-hippo-equivariance}
Fix \(Q\in GL(d)\) independent of \(\tau\), and define
\[
(\mathcal M_{Q,t}u)(\tau):=Q^\top u(\tau).
\]
Then
\[
\operatorname{hippo}_t(\mathcal M_{Q,t}u)
=
\operatorname{hippo}_t(u)\,Q.
\]
\end{lemma}

\begin{proof}
Let \(C_u(t)=\operatorname{hippo}_t(u)\). By the ordinary HiPPO normal equation,
\[
G(t)C_u(t)
=
\int_0^t
\Phi_t(\tau)u(\tau)^\top\,d\mu_t(\tau).
\]
For \(\mathcal M_{Q,t}u\),
\begin{align*}
G(t)C_{\mathcal M_Q u}(t)
&=
\int_0^t
\Phi_t(\tau)(\mathcal M_{Q,t}u)(\tau)^\top\,d\mu_t(\tau)
\\
&=
\int_0^t
\Phi_t(\tau)u(\tau)^\top Q\,d\mu_t(\tau)
\\
&=
G(t)C_u(t)Q.
\end{align*}
Since \(G(t)\) is invertible,
\[
C_{\mathcal M_Q u}(t)=C_u(t)Q.
\]
\end{proof}

\begin{proposition}[Moving-frame factorization of SHiPPO]
\label{prop:appB-moving-frame-factorization}
For a chosen transport path \(A_R\) and its fundamental matrix \(V\), SHiPPO
admits the factorization
\[
\operatorname{shippo}_t
=
\mathcal D_{V,t}\circ\operatorname{hippo}_t\circ\mathcal E_V.
\]
Equivalently, for every signal \(f\),
\[
\operatorname{shippo}_t(f)
=
\operatorname{hippo}_t(\mathcal E_V f)\,V(t).
\]
\end{proposition}

\begin{proof}
By Lemma~\ref{lem:appB-frame-repr},
\[
(\mathcal T_tf)(\tau)
=
P(t,\tau)^\top f(\tau)
=
V(t)^\top V(\tau)^{-T}f(\tau)
=
V(t)^\top(\mathcal E_V f)(\tau).
\]
Thus
\[
\mathcal T_t
=
\mathcal M_{V(t),t}\circ\mathcal E_V.
\]
Using the SHiPPO--HiPPO conjugacy
\[
\operatorname{shippo}_t
=
\operatorname{hippo}_t\circ\mathcal T_t
\]
from Proposition~\ref{prop:app-transport-conjugacy}, and applying
Lemma~\ref{lem:appB-hippo-equivariance} with \(Q=V(t)\), we obtain
\begin{align*}
\operatorname{shippo}_t(f)
&=
\operatorname{hippo}_t
\bigl(
\mathcal M_{V(t),t}(\mathcal E_V f)
\bigr)
\\
&=
\operatorname{hippo}_t(\mathcal E_V f)\,V(t)
\\
&=
\mathcal D_{V,t}
\bigl(
\operatorname{hippo}_t(\mathcal E_V f)
\bigr).
\end{align*}
\end{proof}

\paragraph{Interpretation.}
Proposition~\ref{prop:appB-moving-frame-factorization} shows that, for a chosen transport path, SHiPPO can be viewed as ordinary HiPPO in the moving channel frame \(V(\tau)^{-T}\), followed by decoding with the same frame \(V(t)\).
The same frame determines both the history encoder
\[
(\mathcal E_V f)(\tau)=V(\tau)^{-T}f(\tau)
\]
and the coefficient decoder
\[
\mathcal D_{V,t}(C)=CV(t).
\]
Thus this is a tied, time-varying encoder--decoder factorization, not an equivalence to arbitrary independent channel projections.
When the transport is controller-generated, the same factorization remains valid pathwise; because the frame then depends on the realized input trajectory and model parameters, it does not yield an input-independent encoder--HiPPO--decoder decomposition.

\subsection{Gauge equivalence of coefficient dynamics}
\label{app:gauge-dynamics}
\label{app:gauge-equivalence}

The moving-frame factorization also appears directly at the level of the
coefficient ODEs.

\begin{proposition}[Gauge equivalence of coefficient dynamics]
\label{prop:appB-gauge-dynamics}
Assume \(G(t)=I_N\) and the HiPPO closure condition holds. For a chosen
transport path \(A_R\), let \(V\) be the fundamental matrix from
Definition~\ref{def:appB-global-frame}, and define
\[
\widetilde f(t):=V(t)^{-T}f(t).
\]
Let \(\widetilde C(t)\) be the ordinary HiPPO coefficient trajectory of
\(\widetilde f\):
\[
\dot{\widetilde C}(t)
=
A_L(t)\widetilde C(t)+B_L(t)\widetilde f(t)^\top.
\]
Define
\[
C_S(t):=\widetilde C(t)V(t).
\]
Then \(C_S\) satisfies the SHiPPO Sylvester dynamics
\[
\dot C_S(t)
=
A_L(t)C_S(t)+B_L(t)f(t)^\top+C_S(t)A_R(t).
\]
Conversely, if \(C_S\) satisfies the Sylvester dynamics above, then
\[
\widetilde C(t):=C_S(t)V(t)^{-1}
\]
satisfies the ordinary HiPPO dynamics driven by
\(\widetilde f(t)=V(t)^{-T}f(t)\).
\end{proposition}

\begin{proof}
Differentiate \(C_S=\widetilde C V\):
\[
\dot C_S
=
\dot{\widetilde C}V+\widetilde C\dot V.
\]
Substituting
\[
\dot{\widetilde C}
=
A_L\widetilde C+B_L\widetilde f^\top,
\qquad
\dot V=VA_R,
\]
gives
\begin{align*}
\dot C_S
&=
(A_L\widetilde C+B_L\widetilde f^\top)V
+
\widetilde C V A_R
\\
&=
A_LC_S+B_L(\widetilde f^\top V)+C_SA_R.
\end{align*}
Since
\[
\widetilde f^\top V
=
(V^{-T}f)^\top V
=
f^\top V^{-1}V
=
f^\top,
\]
we obtain
\[
\dot C_S
=
A_LC_S+B_Lf^\top+C_SA_R.
\]

For the converse, define \(\widetilde C=C_SV^{-1}\). Since
\[
\partial_tV^{-1}=-A_RV^{-1},
\]
we have
\begin{align*}
\dot{\widetilde C}
&=
\dot C_SV^{-1}+C_S\partial_tV^{-1}
\\
&=
(A_LC_S+B_Lf^\top+C_SA_R)V^{-1}
-
C_SA_RV^{-1}
\\
&=
A_L\widetilde C+B_L f^\top V^{-1}.
\end{align*}
Finally,
\[
f^\top V^{-1}
=
(V^{-T}f)^\top
=
\widetilde f^\top,
\]
so
\[
\dot{\widetilde C}
=
A_L\widetilde C+B_L\widetilde f^\top.
\]
\end{proof}

\paragraph{What the equivalence does and does not say.}
Propositions~\ref{prop:appB-moving-frame-factorization}
and~\ref{prop:appB-gauge-dynamics} are exact for a chosen transport path, and pathwise for controller-generated transport.
They show that the Sylvester right-action term is precisely the gauge term induced by decoding ordinary HiPPO coefficients from a moving channel frame.

They do not reduce SHiPPO to a single static channel mixer. If \(V(t)\equiv V_0\) is constant, then
\[
0=\dot V(t)=V_0A_R(t),
\]
and since \(V_0\in GL(d)\), this implies \(A_R(t)=0\) almost everywhere. Thus any nontrivial transport requires a genuinely time-varying frame.

Nor do they justify arbitrary independent encoder and decoder maps. The factorization relies on the tied moving-frame relation
\[
(\mathcal E_V f)(\tau)=V(\tau)^{-T}f(\tau),
\qquad
\mathcal D_{V,t}(C)=CV(t),
\]
and this tied structure is exactly what produces the Sylvester term \(C_SA_R\).
In the input-dependent selective setting, the same calculation remains valid pathwise, but the frame itself depends on the realized input trajectory, so there is generally no fixed encoder--HiPPO--decoder decomposition independent of the input.


\section{Proofs for the Scan-Compatible SHiPPO Lift}
\label{app:scan-friendly-shippo}

This appendix supports Section~\ref{sec:selective-transport-cell}. It proves the
obstruction for the direct channelwise lift, the block-affine scan closure of
the restricted realization, the sufficient collapse criterion under fixed
simultaneous block reduction, and the split-flow and discretization claims used
to construct the scan-compatible SHiPPO cell. It does not re-derive the abstract
SHiPPO online approximation operator of Section~\ref{sec:shippo}; that role is
played by Appendices~A and~B.

Unless otherwise stated, we work with a single transport group and omit the
group index. The group-local memory state is
\[
H_t\in\mathbb R^{N\times P},
\]
where \(N\) is the left state size and \(P\) is the group-local channel width.
The left factor is
\[
L_t=\operatorname{Diag}(\alpha_t)\in\mathbb R^{N\times N},
\]
the group-local right action is
\[
R_t\in GL(P),
\]
and the additive source is
\[
U_t\in\mathbb R^{N\times P}.
\]
All statements in this appendix concern the restricted scan-compatible
realization of Section~\ref{sec:selective-transport-cell}. The group-tied
diagonal-left restriction, controller-compatible right transport, and
group-local right actions are computational restrictions of this realization,
not part of the abstract SHiPPO definition.

\subsection{Why the direct channelwise lift does not preserve a small block-affine scan algebra}
\label{app:direct-lift-obstruction}

This subsection justifies the statement in Section~\ref{sec:selective-transport-cell}
that a fully channelwise lift of a diagonal selective recurrence does not
generally close in the small block-affine scan algebra used by the restricted
realization. The direct groupwise lift would be
\begin{equation}
H^{\mathrm{full}}_t
=
(\bar A_t\odot H^{\mathrm{full}}_{t-1})R_t+U_t,
\qquad
\bar A_t\in\mathbb R^{N\times P}.
\label{eq:appC-direct-lift}
\end{equation}
It keeps channel-specific left decays, but nontrivial right transport generally
destroys the shared factorized right summary used in Proposition~\ref{prop:affine-closure}.

\begin{proposition}[Failure of closure for the direct channelwise factorization]
\label{prop:appC-direct-obstruction}
For the direct lift \eqref{eq:appC-direct-lift}, write \(\bar a_{t,n}\in\mathbb R^P\)
for the \(n\)-th row of \(\bar A_t\) and set
\[
D_{t,n}:=\operatorname{Diag}(\bar a_{t,n}).
\]
Then the \(n\)-th row satisfies
\[
h^{\mathrm{full}}_{t,n}=h^{\mathrm{full}}_{t-1,n}D_{t,n}R_t+u_{t,n}.
\]
After two steps,
\[
h^{\mathrm{full}}_{2,n}
=
h^{\mathrm{full}}_{0,n}D_{1,n}R_1D_{2,n}R_2+u_{1,n}D_{2,n}R_2+u_{2,n}.
\]
In general there need not exist a \emph{shared} right factor
\(R_\star\in GL(P)\) and rowwise diagonals \(\widetilde D_n\) such that
\[
D_{1,n}R_1D_{2,n}R_2=\widetilde D_nR_\star
\qquad\text{for all }n.
\]
Hence the direct lift does not, in general, preserve the small block-affine scan
algebra used by the restricted realization.
\end{proposition}

\begin{proof}
The rowwise formula is immediate from
\[
(\bar A_t\odot H)_{n,:}=H_{n,:}D_{t,n}.
\]
Composing two steps gives
\begin{align*}
h^{\mathrm{full}}_{2,n}
&=
\bigl(h^{\mathrm{full}}_{1,n}D_{2,n}\bigr)R_2+u_{2,n}
\\
&=
\bigl(h^{\mathrm{full}}_{0,n}D_{1,n}R_1+u_{1,n}\bigr)D_{2,n}R_2+u_{2,n}
\\
&=
h^{\mathrm{full}}_{0,n}D_{1,n}R_1D_{2,n}R_2+u_{1,n}D_{2,n}R_2+u_{2,n}.
\end{align*}
So closure in the same factorized family would require a common \(R_\star\) and
rowwise diagonals \(\widetilde D_n\) satisfying
\[
D_{1,n}R_1D_{2,n}R_2=\widetilde D_nR_\star
\qquad\text{for all }n.
\]
This need not hold. Consider the concrete example \(P=2\), \(N\ge 2\),
\[
R_1=
\begin{pmatrix}
1&1\\
1&-1
\end{pmatrix},
\qquad
R_2=I_2,
\qquad
D_{1,n}=I_2,
\]
and choose
\[
D_{2,1}=\mathrm{Diag}(1,1),
\qquad
D_{2,2}=\mathrm{Diag}(1,2).
\]
Then
\[
R_1D_{2,1}=
\begin{pmatrix}
1&1\\
1&-1
\end{pmatrix},
\qquad
R_1D_{2,2}=
\begin{pmatrix}
1&2\\
1&-2
\end{pmatrix}.
\]
If there were a shared \(R_\star\) and diagonals \(\widetilde D_1,\widetilde D_2\)
with
\[
R_1D_{2,j}=\widetilde D_jR_\star,
\qquad j\in\{1,2\},
\]
then the first rows of \(R_1D_{2,1}\) and \(R_1D_{2,2}\) would have to be scalar
multiples of the same row of \(R_\star\). But \((1,1)\) and \((1,2)\) are not
scalar multiples, so no such common \(R_\star\) exists.
\end{proof}

\paragraph{Interpretation.}
Proposition~\ref{prop:appC-direct-obstruction} does \emph{not} say that the
direct lift is invalid or unscannable. Rather, it says that the small
factorized scan algebra of the main text is lost. Exact scan is still possible
in larger summaries, for example by carrying rowwise \(P\times P\) transition
matrices. The group-tied diagonal-left realization in
Section~\ref{sec:selective-transport-cell} is introduced precisely to recover a
lightweight factorized block-affine scan.

\subsection{Controller-compatible block-affine scan closure}
\label{app:block-affine-closure-proof}

We now prove Proposition~\ref{prop:affine-closure}. The key point is that
the right action may be token-dependent, but it must be fixed conditional on a
precomputed controller path and independent of the main memory state.

\begin{proof}[Proof of Proposition~\ref{prop:affine-closure}]
Conditional on a causal controller path computed independently of the main
memory recurrence, \(L_t\), \(R_t\), and \(U_t\) are fixed with respect to the
main memory variable \(H_{t-1}\). Hence each step map has the affine two-sided
form
\[
F_t(H)=L_tHR_t+U_t.
\]
Consider two successive steps
\[
F_1(H)=L_1HR_1+U_1,
\qquad
F_2(H)=L_2HR_2+U_2.
\]
Their composition is
\begin{align*}
F_2(F_1(H))
&=
L_2(L_1HR_1+U_1)R_2+U_2
\\
&=
(L_2L_1)H(R_1R_2)+L_2U_1R_2+U_2.
\end{align*}
Therefore the composite update is again of the same affine two-sided form, with
summary
\[
(L_2,R_2,U_2)\star(L_1,R_1,U_1)
=
(L_2L_1,\;R_1R_2,\;L_2U_1R_2+U_2).
\]
Associativity of \(\star\) follows from associativity of function composition,
or directly from associativity of matrix multiplication. The resulting parallel
execution is an associative prefix-scan computation in the standard sense
\citep{blelloch1990prefix,martin2018parallelizing}.
\end{proof}

\paragraph{Why direct state coupling generically breaks the block-affine algebra.}
Suppose instead that the right action depends directly on the main memory,
\[
F_t(H)=L_tH\,R(\xi_t,H)+U_t.
\]
Any affine map \(F\) satisfies, for all scalars \(s\),
\[
F(sH)-F(0)=s\bigl(F(H)-F(0)\bigr).
\]
Here \(F(0)=U_t\), so affine-ness would require
\[
L_t(sH)R(\xi_t,sH)=sL_tHR(\xi_t,H)
\qquad\text{for all }s.
\]
Equivalently,
\[
L_tH\,R(\xi_t,sH)=L_tH\,R(\xi_t,H)
\qquad\text{for all }s.
\]
If there exists an \(H\) such that \(L_tH\) has full column rank, then right
multiplication by \(R\) is injective:
\[
L_tHR_1=L_tHR_2
\quad\Longrightarrow\quad
R_1=R_2.
\]
Thus affine-ness forces
\[
R(\xi_t,sH)=R(\xi_t,H)
\qquad\text{for all }s
\]
along that ray. Therefore any genuine dependence on \(H\), meaning a dependence
for which some \(H\) with \(\mathrm{rank}(L_tH)=P\) and some \(s\neq 1\) satisfy
\[
R(\xi_t,sH)\neq R(\xi_t,H),
\]
makes the step map non-affine. In the group-tied diagonal-left cell, this
full-rank condition holds generically whenever all diagonal entries of \(L_t\)
are nonzero, \(N\ge P\), and \(H\) has full column rank.

\paragraph{Concrete failure mode.}
Let \(W\in\mathbb R^{N\times P}\) and let \(M\in\mathbb R^{P\times P}\) be any
nonzero nilpotent matrix, for example \(M=e_ie_j^\top\) with \(i\neq j\). Define
\[
R(\xi_t,H):=\exp\!\bigl(\kappa\langle W,H\rangle_F M\bigr)
=I_P+\kappa\langle W,H\rangle_F M,
\qquad
\kappa\neq 0,
\]
where the second equality uses \(M^2=0\). Then the update
\[
H_t=L_tH_{t-1}R(\xi_t,H_{t-1})+U_t
\]
contains the quadratic term
\[
\kappa\langle W,H_{t-1}\rangle_F L_tH_{t-1}M,
\]
so it cannot be represented by a finite block-affine summary triple of the form
\((L,R,U)\). This is a limitation of finite block-affine scan summaries, not a
limitation of SHiPPO as an online approximation framework: state-coupled
transport can be defined pathwise, but it does not generically preserve the
finite summary algebra used by the scan-compatible realization.

\subsection{Collapse under fixed simultaneous block reduction}
\label{app:collapse-proof}

We prove Proposition~\ref{prop:collapse}. The result is a sufficient algebraic
collapse criterion: it shows that right-generator families preserving a fixed
common nontrivial block decomposition reduce to independent scalar or blockwise
transported banks after a static change of channel basis. It is not a complete
classification of all possible degenerate transport families.

\begin{proof}[Proof of Proposition~\ref{prop:collapse}]
Assume there exists a fixed \(Q\in GL(P)\) such that
\[
G_m=Q\Lambda_mQ^{-1}
\qquad (m=1,\dots,M)
\]
for all \(m\), where each \(\Lambda_m\) is either diagonal or shares a
common block-diagonal structure with the same fixed nontrivial block partition.
Define
\[
A_{R,t}=
\sum_{m=1}^M \rho_{t,m}G_m.
\]
Then
\[
A_{R,t}
=
Q\left(\sum_{m=1}^M\rho_{t,m}\Lambda_m\right)Q^{-1}
=
Q\Lambda_tQ^{-1},
\qquad
\Lambda_t:=\sum_{m=1}^M\rho_{t,m}\Lambda_m.
\]
By construction, \(\Lambda_t\) is diagonal, or has the same fixed block-diagonal
structure.

For the dense exponential choice,
\[
R_t=\exp(\Delta_tA_{R,t}),
\]
we obtain
\[
R_t=Q\exp(\Delta_t\Lambda_t)Q^{-1}.
\]
The matrix \(\exp(\Delta_t\Lambda_t)\) is diagonal in the diagonal case and
block-diagonal with the same fixed blocks in the block case.

The same conclusion holds for any fixed-order split product formed from the same
generators. Indeed, each factor satisfies
\[
\exp(\Delta_t\rho_{t,m}G_m)
=
Q\exp(\Delta_t\rho_{t,m}\Lambda_m)Q^{-1}.
\]
Thus any fixed-order product of such factors has the form
\[
\prod_{m=1}^M \exp(\Delta_t\rho_{t,m}G_m)
=
Q\left(\prod_{m=1}^M \exp(\Delta_t\rho_{t,m}\Lambda_m)\right)Q^{-1},
\]
with the products taken in the implemented fixed order. The middle product
remains diagonal or block-diagonal with the same fixed block structure.

Now define the transformed state
\[
\widetilde H_t:=H_tQ,
\qquad
\widetilde U_t:=U_tQ.
\]
The recurrence
\[
H_t=L_tH_{t-1}R_t+U_t
\]
then becomes
\[
\widetilde H_t
=
L_t\widetilde H_{t-1}\widetilde R_t+\widetilde U_t,
\qquad
\widetilde R_t:=Q^{-1}R_tQ,
\]
where \(\widetilde R_t\) is diagonal or block-diagonal with the same fixed block
partition.

If \(\widetilde R_t\) is diagonal, the columns of \(\widetilde H_t\) evolve
independently, so the recurrence decomposes into independent scalar transported
banks. If \(\widetilde R_t\) is block-diagonal with a fixed partition, then
columns evolve independently across blocks, with mixing only within each block,
yielding independent blockwise transported banks.

This proves the recurrence-level collapse. The fixed basis change \(Q\) appears
only as a static change of channel coordinates: it can be absorbed into
surrounding static projections or treated as a one-time change of channel basis.
In particular, if the source is factorized as \(U_t=b_t x_t^\top\), then
\[
U_tQ=b_t(Q^\top x_t)^\top,
\]
so the change of basis is absorbed into the group-local input coordinates. Thus
the recurrence collapses, after the static basis change \(Q\), to independent
scalar or fixed-block transported banks. This proves the sufficient collapse
criterion in Proposition~\ref{prop:collapse}: token-dependent coefficients over a
simultaneously reducible generator family do not yield genuinely coupled
transported memory across the fixed blocks.
\end{proof}

\subsection{Closed-form right actions for split-flow factors}
\label{app:cheap-right-actions}

This subsection records closed-form exponentials and right-actions for the
structured factors used in the split-flow generator
\eqref{eq:split-flow-generator}. We use matrix exponentials in the standard
matrix-function sense \citep{higham2008functions}. These formulas define the
optional structured backend used to construct the per-token action \(R_t\). The
dense frozen-ODE backend uses
\[
R_t^{\mathrm{den}}=\exp(\Delta_tA_{R,t})
\]
and is always invertible. A split backend defines the implemented right action
\(R_t^{\mathrm{split}}\) as a product of factor exponentials and also preserves
invertibility because each factor below is invertible.

For the dissipative diagonal component in \eqref{eq:split-flow-generator}, the
diagonal factor is
\[
D_t^{\mathrm{diag}}
=
-\Delta_t\operatorname{Diag}(d_t),
\qquad
d_t\in\mathbb R_+^P.
\]
Hence
\[
\exp(D_t^{\mathrm{diag}})
=
\operatorname{Diag}\bigl(e^{-\Delta_t d_t}\bigr),
\]
whose entries lie in \((0,1]\) for \(\Delta_t\ge0\). This is the sense in which
the diagonal component is dissipative. The proposition below records the more
general diagonal scaling formula used by the implementation.

\begin{proposition}[Closed forms and cheap right-actions]
\label{prop:cheap-actions}
Let \(H\in\mathbb R^{N\times P}\). The following right-actions have closed
forms.
\begin{enumerate}
\item \textbf{Diagonal scalings.}
If \(D=\mathrm{Diag}(\delta)\), then
\[
\exp(D)=\mathrm{Diag}(\exp(\delta)).
\]
Thus \(H\exp(D)\) costs \(O(NP)\), and the inverse is \(\exp(-D)\).

\item \textbf{Two-plane rotations.}
Fix distinct indices \(i\neq j\), and define the skew-symmetric generator
\[
\Omega_{ij}:=e_je_i^\top-e_ie_j^\top.
\]
Then \(\Omega_{ij}^\top=-\Omega_{ij}\), and
\(\exp(\phi\Omega_{ij})\) acts nontrivially only on columns \(i\) and \(j\):
\[
\begin{pmatrix}
H_{:,i} & H_{:,j}
\end{pmatrix}
\mapsto
\begin{pmatrix}
H_{:,i} & H_{:,j}
\end{pmatrix}
\begin{pmatrix}
\cos\phi & -\sin\phi\\
\sin\phi & \cos\phi
\end{pmatrix}.
\]
The right-action costs \(O(N)\). The inverse is
\[
\exp(-\phi\Omega_{ij})=\exp(\phi\Omega_{ij})^\top.
\]

\item \textbf{Nilpotent shears.}
Fix \(i\neq j\), and define
\[
N_{ij}:=e_ie_j^\top.
\]
Then \(N_{ij}^2=0\), so
\[
\exp(\eta N_{ij})=I+\eta N_{ij}.
\]
Moreover,
\[
H\exp(\eta N_{ij})=H+\eta HN_{ij},
\qquad
HN_{ij}=He_i e_j^\top,
\]
which adds \(\eta\) times column \(i\) into column \(j\). The cost is \(O(N)\),
and the inverse is \(I-\eta N_{ij}\).

\item \textbf{Rank-one exponentials.}
Let \(A=uv^\top\in\mathbb R^{P\times P}\) be rank one. Since
\[
A^2=(v^\top u)A,
\]
for any scalar \(s\),
\[
\exp(suv^\top)=I+\varphi(s\,v^\top u)\,suv^\top,
\]
where
\[
\varphi(\kappa)
:=
\begin{cases}
\dfrac{e^\kappa-1}{\kappa}, & \kappa\neq 0,\\[0.75em]
1, & \kappa=0.
\end{cases}
\]
Consequently,
\[
H\exp(suv^\top)=H+\bigl(\varphi(s\,v^\top u)\,s\bigr)(Hu)v^\top.
\]
The right-action costs \(O(NP)\). The inverse is obtained by replacing \(s\)
with \(-s\).
\end{enumerate}
\end{proposition}

\begin{proof}
The diagonal case is immediate from entrywise exponentiation.

For the two-plane rotation, the restriction of \(\Omega_{ij}\) to the
\((i,j)\)-plane is
\[
\begin{pmatrix}
0 & -1\\
1 & 0
\end{pmatrix},
\]
whose exponential is
\[
\begin{pmatrix}
\cos\phi & -\sin\phi\\
\sin\phi & \cos\phi
\end{pmatrix}.
\]
All other coordinates are fixed.

For the nilpotent shear, \(N_{ij}^2=0\), so the exponential series truncates:
\[
\exp(\eta N_{ij})=I+\eta N_{ij}.
\]
Right multiplication by \(N_{ij}\) maps column \(i\) into column \(j\), giving
the stated update.

For the rank-one exponential, the identity
\[
(uv^\top)^k=(v^\top u)^{k-1}uv^\top
\qquad (k\ge 1)
\]
gives the desired formula. If \(v^\top u\neq 0\), then
\[
\exp(suv^\top)
=
I+
\sum_{k=1}^{\infty}\frac{s^k(v^\top u)^{k-1}}{k!}uv^\top
=
I+\frac{e^{sv^\top u}-1}{v^\top u}uv^\top
=
I+\varphi(sv^\top u)\,suv^\top.
\]
If \(v^\top u=0\), then \((uv^\top)^2=0\), so
\[
\exp(suv^\top)=I+suv^\top,
\]
which is the same formula with the continuous extension \(\varphi(0)=1\).
Right multiplying \(H\) gives the displayed formula.
\end{proof}

\paragraph{Numerical note.}
For small \(|\kappa|\), the scalar
\[
\varphi(\kappa)=\frac{e^\kappa-1}{\kappa}
\]
should be evaluated with a numerically stable \(\mathrm{expm1}\)-style
implementation.

\subsection{Lie--Trotter split products}
\label{app:lie-trotter}

A structured backend may define the discrete right action \(R_t\) as a fixed-order
product of cheap factor exponentials. This product is the actual right action
used by the discrete recurrence \eqref{eq:discrete-shippo-cell}. Hence the
block-affine scan algebra remains exact for that chosen recurrence. The lemma
below quantifies only the local approximation error of the split product
relative to the dense exponential \(\exp(\Delta_tA_{R,t})\), as in standard
first-order splitting methods for differential equations
\citep{hairer2006geometric}.

\begin{lemma}[First-order Lie--Trotter local error]
\label{lem:lie-trotter}
Let
\[
A=\sum_{k=1}^K A_k,
\]
and define the fixed-order product
\[
\widetilde R(\Delta)
:=
\exp(\Delta A_1)\exp(\Delta A_2)\cdots\exp(\Delta A_K).
\]
If \(\|A_k\|\le M\) for all \(k\), then for sufficiently small \(\Delta\),
\[
\bigl\|\widetilde R(\Delta)-\exp(\Delta A)\bigr\|
\le
C(K,M)\Delta^2,
\]
where \(C(K,M)\) depends only on \(K\) and \(M\).
\end{lemma}

\begin{proof}
Use a submultiplicative matrix norm. For \(\Delta\le 1\), Taylor's theorem gives
\[
\exp(\Delta A_k)=I+\Delta A_k+E_k(\Delta),
\qquad
\|E_k(\Delta)\|\le c(M)\Delta^2
\]
for a constant \(c(M)\) depending only on \(M\). Multiplying the \(K\) factors
in this fixed order, all first-order terms add and all products containing at
least two first-order or remainder terms are
\(O_{K,M}(\Delta^2)\). Hence
\[
\widetilde R(\Delta)
=
I+\Delta\sum_{k=1}^K A_k+E_{\mathrm{prod}}(\Delta),
\qquad
\|E_{\mathrm{prod}}(\Delta)\|\le c_1(K,M)\Delta^2.
\]
Similarly,
\[
\exp(\Delta A)=I+\Delta A+E_{\mathrm{exp}}(\Delta),
\qquad
\|E_{\mathrm{exp}}(\Delta)\|\le c_2(K,M)\Delta^2,
\]
because
\[
\|A\|\le \sum_{k=1}^K\|A_k\|\le KM.
\]
Since \(A=\sum_k A_k\), subtracting the two expansions gives
\[
\bigl\|\widetilde R(\Delta)-\exp(\Delta A)\bigr\|
\le
\bigl(c_1(K,M)+c_2(K,M)\bigr)\Delta^2.
\]
Taking \(C(K,M)=c_1(K,M)+c_2(K,M)\) proves the claim.
\end{proof}

\subsection{Proof of Theorem~\ref{thm:discrete-cell}: exponential-adjusted discrete cell}
\label{app:discrete-cell-proof}

We now prove Theorem~\ref{thm:discrete-cell}. The derivation below establishes
only the dense-backend frozen-ODE interpretation, corresponding to the choice
\[
R_t=R_t^{\mathrm{den}}=\exp(\Delta_t A_{R,t}).
\]
If a structured backend instead defines \(R_t=R_t^{\mathrm{split}}\) by a
Lie--Trotter product, the same algebraic update is exact for that chosen
split-flow recurrence, while Lemma~\ref{lem:lie-trotter} describes only the
product's local approximation to the dense exponential flow. A comparison of
the split-backend update with the dense frozen ODE must therefore add the
split-backend replacement error recorded in
Remark~\ref{rem:split-backend-replacement-error}.

\begin{proof}[Proof of Theorem~\ref{thm:discrete-cell}]
Consider one step of length \(\Delta_t>0\). For the local discretization
analysis, freeze the left drift \(a_t\in\mathbb R^N\) and the right generator
\(A_{R,t}\in\mathbb R^{P\times P}\) on the interval \([t-\Delta_t,t]\). Let the
source vary over the step:
\[
U(\tau)\in\mathbb R^{N\times P}.
\]
The step-frozen restricted dynamics are
\begin{equation}
\dot H(\tau)=\operatorname{Diag}(a_t)H(\tau)+U(\tau)+H(\tau)A_{R,t},
\qquad
\tau\in[t-\Delta_t,t].
\label{eq:appC-step-frozen-shippo}
\end{equation}
Define
\[
L_t:=\exp\!\bigl(\Delta_t\operatorname{Diag}(a_t)\bigr)
=\operatorname{Diag}(\alpha_t),
\qquad
R_t^{\mathrm{den}}:=\exp(\Delta_tA_{R,t}),
\]
where \(\alpha_t:=\exp(\Delta_t a_t)\) entrywise. By variation of constants for
the two-sided linear equation \eqref{eq:appC-step-frozen-shippo}, the exact dense
frozen-ODE step is
\begin{equation}
H_t^\star
=
L_tH_{t-1}R_t^{\mathrm{den}}
+
\int_0^{\Delta_t}
\exp\!\bigl((\Delta_t-s)\operatorname{Diag}(a_t)\bigr)
U(t-\Delta_t+s)
\exp\!\bigl((\Delta_t-s)A_{R,t}\bigr)
\,ds,
\label{eq:appC-duhamel-sylvester}
\end{equation}
where
\[
H_{t-1}:=H(t-\Delta_t),
\qquad
H_t^\star:=H(t).
\]
Indeed, this follows by multiplying on the left by
\(\exp(- (\tau-(t-\Delta_t))\operatorname{Diag}(a_t))\) and on the right by
\(\exp(- (\tau-(t-\Delta_t))A_{R,t})\), differentiating the transformed
quantity, and integrating over the interval.

Define the Duhamel integrand
\[
G(s)
:=
\exp\!\bigl((\Delta_t-s)\operatorname{Diag}(a_t)\bigr)
U(t-\Delta_t+s)
\exp\!\bigl((\Delta_t-s)A_{R,t}\bigr),
\qquad
s\in[0,\Delta_t].
\]
Using endpoint source samples
\[
U_{t-1}=U(t-\Delta_t),
\qquad
U_t=U(t),
\]
or discrete approximations interpreted as these endpoint source values, we have
\[
G(0)=L_tU_{t-1}R_t^{\mathrm{den}},
\qquad
G(\Delta_t)=U_t.
\]
For \(t=1\) with \(\beta_1\neq 0\), the left endpoint requires an explicit
boundary source value representing \(U(0)\); the implementation convention
\(U_0:=0\) corresponds to the assumption \(U(0)=0\) or to an intended zero
boundary source.

Approximate the integral in \eqref{eq:appC-duhamel-sylvester} by the two-point
family
\[
\int_0^{\Delta_t}G(s)\,ds
\approx
(1-\lambda_t)\Delta_t G(0)+\lambda_t\Delta_t G(\Delta_t),
\qquad
\lambda_t\in[0,1].
\]
This gives
\[
\int_0^{\Delta_t}G(s)\,ds
\approx
(1-\lambda_t)\Delta_t L_tU_{t-1}R_t^{\mathrm{den}}
+
\lambda_t\Delta_t U_t.
\]
Define the dense-backend source update
\[
\widehat U_t^{\mathrm{den}}
:=
(1-\lambda_t)\Delta_t L_tU_{t-1}R_t^{\mathrm{den}}+
\lambda_t\Delta_t U_t.
\]
The corresponding dense discrete update is
\[
H_t^{\mathrm{den}}
:=
L_tH_{t-1}R_t^{\mathrm{den}}+\widehat U_t^{\mathrm{den}},
\]
which is exactly the dense-backend instance of the discrete cell
\eqref{eq:discrete-shippo-cell}. For a split backend, the same algebraic cell is
executed with the implemented right action \(R_t^{\mathrm{split}}\), but that is
a statement about the chosen discrete recurrence rather than the dense
frozen-ODE Duhamel step.

If \(\lambda_t=1\), then
\[
\widehat U_t^{\mathrm{den}}=\Delta_t U_t,
\]
so the source is evaluated at the right endpoint. This is the exponential-Euler
source rule. If \(\lambda_t=\frac12+O(\Delta_t)\), the local quadrature behavior
is the generalized exponential-trapezoidal behavior stated in
Corollary~\ref{cor:local-truncation} below.

Finally, if \(A_{R,t}\equiv0\), then \(R_t^{\mathrm{den}}=I_P\); any split
backend formed from zero right-generator factors also gives the identity right
action. The update becomes
\[
H_t
=
L_tH_{t-1}
+
(1-\lambda_t)\Delta_t L_tU_{t-1}
+
\lambda_t\Delta_t U_t.
\]
This is the corresponding one-sided group-tied diagonal-left selective update.
If, in addition, \(P=1\) and \(\lambda_t=1\), then
\[
H_t=L_tH_{t-1}+\Delta_t U_t,
\]
which is the one-channel exponential-Euler selective update associated with the
diagonal-left recurrence. For \(\lambda_t\neq1\), the neutral right-action limit
is still one-sided, but it uses the two-point source rule above rather than the
standard right-endpoint exponential-Euler source rule.
\end{proof}

\subsection{Local truncation behavior of the source quadrature}
\label{app:source-quadrature-truncation}

Theorem~\ref{thm:discrete-cell} uses a two-point quadrature rule for the dense
Duhamel source integral. This appendix-level corollary records the local
quadrature details omitted from Section~\ref{sec:discrete-shippo-cell} for
space. It concerns source quadrature relative to the step-frozen continuous
dynamics, not the exact block-affine scan of the resulting discrete recurrence.

\begin{corollary}[Dense-backend local truncation behavior under frozen coefficients]
\label{cor:local-truncation}
Assume \(R_t=R_t^{\mathrm{den}}=\exp(\Delta_tA_{R,t})\), assume \(a_t\) and
\(A_{R,t}\) are frozen on \([t-\Delta_t,t]\), and assume the Duhamel integrand
\(G\) is sufficiently smooth. Also assume that the discrete endpoint source
samples represent the continuous endpoint values,
\[
U_{t-1}=U(t-\Delta_t),
\qquad
U_t=U(t).
\]
For \(t=1\) with hard-wired \(U_0=0\) and \(\beta_1\neq 0\), this means assuming
\(U(0)=0\) or an explicit zero boundary source for the left endpoint. Let
\(H_t^\star\) denote the exact step map \eqref{eq:appC-duhamel-sylvester}, and let
\(H_t^{\mathrm{den}}\) denote the dense discrete update from
Theorem~\ref{thm:discrete-cell}. Then:

\begin{enumerate}
\item For any \(\lambda_t\in[0,1]\),
\[
\|H_t^{\mathrm{den}}-H_t^\star\|
\le
\frac{\Delta_t^2}{2}
\sup_{s\in[0,\Delta_t]}\|G'(s)\|.
\]
Thus the method is locally first-order accurate in \(\Delta_t\).

\item If \(\lambda_t=\frac12\) and \(G\in C^2\), then
\[
\|H_t^{\mathrm{den}}-H_t^\star\|
\le
\frac{\Delta_t^3}{12}
\sup_{s\in[0,\Delta_t]}\|G''(s)\|.
\]
Thus the method is locally second-order accurate in \(\Delta_t\).

\item More generally, if
\[
\left|\lambda_t-\frac12\right|\le c\Delta_t
\]
for a constant \(c\), and \(G\in C^2\), then
\[
\|H_t^{\mathrm{den}}-H_t^\star\|
\le
\Delta_t^3\left(
\frac{1}{12}\sup_{s\in[0,\Delta_t]}\|G''(s)\|
+
c\sup_{s\in[0,\Delta_t]}\|G'(s)\|
\right).
\]
Thus \(\lambda_t=\frac12+O(\Delta_t)\) preserves the same local
\(O(\Delta_t^3)\) quadrature-error order.
\end{enumerate}
\end{corollary}

\begin{proof}[Proof of Corollary~\ref{cor:local-truncation}]
Let
\[
I(G):=\int_0^{\Delta_t}G(s)\,ds,
\qquad
Q_\lambda(G):=(1-\lambda_t)\Delta_t G(0)+\lambda_t\Delta_t G(\Delta_t).
\]
By the definitions of \(H_t^\star\) and \(H_t^{\mathrm{den}}\),
\[
H_t^\star-H_t^{\mathrm{den}}=I(G)-Q_\lambda(G).
\]

For the first claim, write
\[
I(G)-Q_\lambda(G)
=
(1-\lambda_t)\bigl(I(G)-\Delta_tG(0)\bigr)
+\lambda_t\bigl(I(G)-\Delta_tG(\Delta_t)\bigr).
\]
Using
\[
G(s)-G(0)=\int_0^s G'(r)\,dr
\]
gives
\[
\left\|I(G)-\Delta_tG(0)\right\|
\le
\int_0^{\Delta_t}s\,ds\,
\sup_{r\in[0,\Delta_t]}\|G'(r)\|
=
\frac{\Delta_t^2}{2}
\sup_{r\in[0,\Delta_t]}\|G'(r)\|.
\]
Similarly,
\[
\left\|I(G)-\Delta_tG(\Delta_t)\right\|
\le
\frac{\Delta_t^2}{2}
\sup_{r\in[0,\Delta_t]}\|G'(r)\|.
\]
Since \(\lambda_t\in[0,1]\), the first bound follows.

For \(\lambda_t=\frac12\), \(Q_\lambda\) is the trapezoidal rule. The standard
trapezoidal remainder for a Banach-space-valued \(C^2\) function gives
\[
\|I(G)-Q_{1/2}(G)\|
\le
\frac{\Delta_t^3}{12}
\sup_{s\in[0,\Delta_t]}\|G''(s)\|.
\]

For the final claim, write
\[
Q_{\lambda_t}(G)-Q_{1/2}(G)
=
\left(\lambda_t-\frac12\right)
\Delta_t\bigl(G(\Delta_t)-G(0)\bigr).
\]
Therefore, if
\[
\left|\lambda_t-\frac12\right|\le c\Delta_t,
\]
then
\[
\|Q_{\lambda_t}(G)-Q_{1/2}(G)\|
\le
c\Delta_t^2\int_0^{\Delta_t}\|G'(s)\|\,ds
\le
c\Delta_t^3\sup_{s\in[0,\Delta_t]}\|G'(s)\|.
\]
Combining this with the trapezoidal bound proves the result.
\end{proof}

\begin{remark}[Split-backend replacement error]
\label{rem:split-backend-replacement-error}
Let \(H_t^\star\) denote the exact one-step solution of the dense frozen ODE
\eqref{eq:appC-duhamel-sylvester}.  Write
\(R_t^{\mathrm{den}}:=\exp(\Delta_tA_{R,t})\), and let
\(R_t^{\mathrm{split}}\) be the right action produced by a split backend.  For the
same one-step input \(H_{t-1},U_{t-1},U_t\), define
\[
\begin{aligned}
H_t^{\mathrm{den}}
&:=L_tH_{t-1}R_t^{\mathrm{den}}
  +\beta_tL_tU_{t-1}R_t^{\mathrm{den}}+\gamma_tU_t,\\
H_t^{\mathrm{split}}
&:=L_tH_{t-1}R_t^{\mathrm{split}}
  +\beta_tL_tU_{t-1}R_t^{\mathrm{split}}+\gamma_tU_t.
\end{aligned}
\]
Then
\[
\begin{aligned}
H_t^{\mathrm{split}}-H_t^{\mathrm{den}}
&=L_t\bigl(H_{t-1}+\beta_tU_{t-1}\bigr)
  \bigl(R_t^{\mathrm{split}}-R_t^{\mathrm{den}}\bigr),\\
\|H_t^{\mathrm{split}}-H_t^\star\|
&\le
\|L_t\|\,\bigl(\|H_{t-1}\|+|\beta_t|\|U_{t-1}\|\bigr)
\|R_t^{\mathrm{split}}-R_t^{\mathrm{den}}\|
+\|H_t^{\mathrm{den}}-H_t^\star\|.
\end{aligned}
\]
The first term is the split-backend replacement error, while the second term is
the dense-backend source-quadrature error controlled in
Corollary~\ref{cor:local-truncation}. For a first-order Lie--Trotter backend,
Lemma~\ref{lem:lie-trotter} gives
\(\|R_t^{\mathrm{split}}-R_t^{\mathrm{den}}\|=O(\Delta_t^2)\) under bounded
factor generators. Thus, for bounded one-step inputs and bounded \(L_t\), the
full one-step discrepancy to the dense frozen ODE is generically
\(O(\Delta_t^2)\), even when the dense source quadrature has
\(O(\Delta_t^3)\) local error.
\end{remark}

\section{Formal Diagnostics and Experimental Protocol}
\label{app:formal-diagnostics}

This appendix specifies the controlled diagnostics used in
Section~\ref{sec:experiments}. The purpose of these experiments is narrow:
to separate high-rank source/write capacity from future right transport of
memory that has already been written. The diagnostics are synthetic by design,
but they remove direct source-injection shortcuts and expose the mechanism that
is central to the paper.

\subsection{Diagnostic scope}
\label{app:formal-manifest}

The paper-facing diagnostics have four roles. The paired noncommutative
transport diagnostic is the main separation test: it isolates whether high-rank
current-step source writes can replace future right transport of already-written
memory. The right-transport parameterization/intervention study is also main
evidence: it tests oracle, learned, and selective right transports together with
the \(R_t\!\to I\) evaluation-time intervention. The transported-projection
prediction task is auxiliary, because it checks a related transported-projection
quantity but does not include the paired-difference intervention metric. The
group-local audit is a limitation diagnostic, used only to clarify how the
scan-compatible group-local restriction differs from full right transport.

\subsection{Paired noncommutative transport task}
\label{app:paired-transport-task}

The primary diagnostic tests whether high-rank current-step source writes can
replace future right transport of memory written earlier. Each example first
writes a payload vector \(v\in\mathbb{R}^d\), then applies two operation tokens.
The paired examples share the same payload and operation parameters but reverse
the operation order, \(ab\) versus \(ba\). In the minimal setting used for the
main diagnostic, the two right actions are generated by the elementary shears
\begin{equation}
  R_a(\alpha) = I + \alpha e_1 e_2^{\top},
  \qquad
  R_b(\beta) = I + \beta e_2 e_3^{\top},
  \label{eq:app-paired-generators}
\end{equation}
so that
\begin{equation}
  R_a(\alpha)R_b(\beta)-R_b(\beta)R_a(\alpha)
  = \alpha\beta e_1 e_3^{\top}.
  \label{eq:app-paired-commutator}
\end{equation}
The paired target is
\begin{equation}
  \Delta_{\mathrm{true}}
  = v^{\top}\bigl(R_aR_b-R_bR_a\bigr),
  \label{eq:app-paired-delta}
\end{equation}
and the model prediction \(\Delta_{\mathrm{pred}}\) is the difference between
its outputs on the two paired examples. The main metric is
\begin{equation}
  \operatorname{Pair}\Delta\operatorname{NMSE}
  =
  \frac{\lVert \Delta_{\mathrm{pred}}-\Delta_{\mathrm{true}}\rVert_2^2}
       {\max\{\lVert \Delta_{\mathrm{true}}\rVert_2^2,\epsilon_{\mathrm{den}}\}},
  \qquad \epsilon_{\mathrm{den}}=10^{-8}.
  \label{eq:app-pair-delta-nmse}
\end{equation}
A value near one means that the model predicts essentially no paired
noncommutative difference; a value near zero means that it recovers the
transport-order difference. The denominator floor is fixed before training and
is used only to avoid numerical instability on near-zero paired differences.

The key constraint is that source injection is allowed only at WRITE
tokens. Operation tokens do not inject payload-dependent source. Thus a
no-right rank-\(r\) source/write model can increase the rank of what it writes at the
current step, but it cannot implement a later operation that right-multiplies
already-written channel coordinates. This is the intended separation from
SHiPPO-style right transport.

The task schematic and aggregate results are shown in
Figure~\ref{fig:main-evidence}; this appendix gives the formal task definition
and sampling details.

\paragraph{Sampling.}
For the paired-transport runs, \(d=4\), the state size is \(N=16\), and the
operation family has \(K=2\) generators. Payloads are sampled as
\(v\sim \mathcal{N}(0,I_d/d)\), and operation coefficients are sampled uniformly
from \([-0.5,0.5]\). The paired evaluation uses the same \(v,\alpha,\beta\) for
the two orders, so that the measured difference isolates the order-dependent
commutator term.

\subsection{Model families and interventions}
\label{app:model-families}

All formal diagnostics use a controlled matrix-state family with memory state
\(H_t\in\mathbb{R}^{N\times d}\). The no-right rank-\(r\) source/write baseline uses
\begin{equation}
  H_t = L_t H_{t-1} + \widehat U_t^{(r)},
  \qquad
  \widehat U_t^{(r)} = \frac{1}{\sqrt r}B_tX_t^{\top},
  \label{eq:app-no-right}
\end{equation}
where increasing \(r\) increases the rank of the current source update. This is
the controlled high-rank write/read baseline.

SHiPPO-style variants add a right transport:
\begin{equation}
  H_t = L_t H_{t-1}R_t + \widehat U_t.
  \label{eq:app-shippo-right}
\end{equation}
The right action is either oracle-provided, learned as a Lie-style generator
parameterization, or generated by a selective controller. We evaluate
oracle-\(R\), learned-\(R\) initialized from true, zero, or random generators,
and selective-\(R\). For right-transport models we also report an evaluation-time
\(R_t\!\to\!I\) intervention: the right action is replaced by identity while all
other learned weights remain fixed. If the paired-difference signal disappears
under this intervention, the behavior is mediated by right transport rather than
by a static readout or source-only shortcut.

\subsection{Training and evaluation protocol}
\label{app:training-protocol}

All reported values are means and sample standard deviations over independent
training seeds. Training uses AdamW \citep{loshchilov2019decoupled} with learning rate \(3\times10^{-4}\),
weight decay \(10^{-2}\), gradient clipping at \(1.0\), normalized MSE loss,
and evaluation every 1000 steps. We report the final training step; no
validation-based checkpoint selection is used for the paper-facing tables.
The transported-projection prediction task uses generator seed 321,
coefficient range 0.15, and \(\varepsilon_{\mathrm{proj}}=0.05\) in the task
generator. The paired-transport diagnostics use coefficient range 0.5.
Table~\ref{tab:formal-training-settings} records the main settings.

\begin{table}[!htbp]
\centering
\small
\setlength{\tabcolsep}{4pt}
\renewcommand{\arraystretch}{1.10}
\caption{Training and evaluation settings for the formal diagnostics. All runs
use \(20{,}000\) training steps. We use \(n=5\) seeds except for the
group-local audit, which uses \(n=3\). The ``Eval batches'' column records the
number of evaluation batches.}
\label{tab:formal-training-settings}
\begin{tabular}{@{}L{0.28\linewidth}ccL{0.48\linewidth}@{}}
\toprule
Diagnostic & Batch & \makecell[c]{Eval\\batches} & Key settings \\
\midrule
Paired noncommutative transport
& 128 & 16
& \(d=4\), \(N=16\), \(K=2\), depth 2; WRITE-only source gate; pair-loss weight 5.0; no-right ranks \(r\in\{1,2,4,8\}\). \\

Right-transport parameterizations/intervention
& 128 & 16
& Same paired-transport task; compares no-right \(r=8\), oracle-\(R\), learned-\(R\), and selective-\(R\). \\

Transported-projection prediction
& 32 & 8
& \(d=16\), \(N=16\), \(K=16\), length 256; skew--nilpotent generator family; source rank \(r=4\), fixed left operator. \\

Group-local audit
& 128 & 8
& Paired-transport task with group sizes 2 and 4; oracle-\(R\) and learned-\(R\) true initialization; audit batch size 512. \\
\bottomrule
\end{tabular}
\end{table}
\FloatBarrier

\subsection{Full numerical summaries}
\label{app:full-formal-results}

Tables~\ref{tab:paired-transport}--\ref{tab:transported-projection} give the
numerical values underlying Figure~\ref{fig:main-evidence} and the auxiliary
transported-projection result. All entries in these three tables are means
\(\pm\) sample standard deviations over \(n=5\) independent training seeds.

\begin{table}[!htbp]
\centering
\footnotesize
\setlength{\tabcolsep}{2.2pt}
\renewcommand{\arraystretch}{1.12}
\caption{Paired noncommutative transport diagnostic. Entries are means \(\pm\) sample standard deviations over \(n=5\) seeds.}
\label{tab:paired-transport}
\label{tab:e0-formal-separation}
\begin{tabular}{@{}L{0.29\linewidth}cccc@{}}
\toprule
Model &
\makecell[c]{Eval\\NMSE \(\downarrow\)} &
\makecell[c]{Pair \(\Delta\)\\NMSE \(\downarrow\)} &
\makecell[c]{Pair \(\Delta\)\\\(R^2\) \(\uparrow\)} &
\makecell[c]{\(R\!\to\!I\) Pair \(\Delta\)\\NMSE \(\downarrow\)} \\
\midrule
Oracle solver & \ms{0}{0} & \ms{0}{0} & \ms{1}{0} & -- \\
No-right \(r=1\) &
\ms{4.145\!\times\!10^{-4}}{5.33\!\times\!10^{-5}} &
\ms{1}{4.07\!\times\!10^{-7}} &
\ms{-1.52\!\times\!10^{-3}}{2.165\!\times\!10^{-3}} & -- \\
No-right \(r=2\) &
\ms{3.568\!\times\!10^{-4}}{4.38\!\times\!10^{-5}} &
\ms{1}{3.22\!\times\!10^{-7}} &
\ms{-1.52\!\times\!10^{-3}}{2.166\!\times\!10^{-3}} & -- \\
No-right \(r=4\) &
\ms{3.389\!\times\!10^{-4}}{4.43\!\times\!10^{-5}} &
\ms{1}{3.11\!\times\!10^{-7}} &
\ms{-1.52\!\times\!10^{-3}}{2.165\!\times\!10^{-3}} & -- \\
No-right \(r=8\) &
\ms{2.797\!\times\!10^{-4}}{2.31\!\times\!10^{-5}} &
\ms{1}{2.02\!\times\!10^{-7}} &
\ms{-1.52\!\times\!10^{-3}}{2.165\!\times\!10^{-3}} & -- \\
SHiPPO oracle-\(R\) &
\ms{4.377\!\times\!10^{-3}}{9.788\!\times\!10^{-3}} &
\ms{2.44\!\times\!10^{-7}}{4.61\!\times\!10^{-7}} &
\ms{1}{4.66\!\times\!10^{-7}} & \ms{1}{0} \\
SHiPPO learned-\(R\) (true init) &
\ms{7.019\!\times\!10^{-4}}{1.18\!\times\!10^{-3}} &
\ms{2.11\!\times\!10^{-5}}{2.86\!\times\!10^{-5}} &
\ms{1}{2.86\!\times\!10^{-5}} & \ms{1}{0} \\
\bottomrule
\end{tabular}
\end{table}

\begin{table}[!htbp]
\centering
\footnotesize
\setlength{\tabcolsep}{2.2pt}
\renewcommand{\arraystretch}{1.12}
\caption{Right-transport parameterizations and evaluation-time intervention.
Entries are means \(\pm\) sample standard deviations over \(n=5\) seeds. Rows
marked \(\dagger\) are repeated from Table~\ref{tab:paired-transport}.}
\label{tab:right-transport-parameterizations}
\label{tab:rdesign}
\begin{tabular}{@{}L{0.32\linewidth}cccc@{}}
\toprule
Model &
\makecell[c]{Eval\\NMSE \(\downarrow\)} &
\makecell[c]{Pair \(\Delta\)\\NMSE \(\downarrow\)} &
\makecell[c]{Pair \(\Delta\)\\\(R^2\) \(\uparrow\)} &
\makecell[c]{\(R\!\to\!I\) Pair \(\Delta\)\\NMSE \(\downarrow\)} \\
\midrule
No-right \(r=8\)\textsuperscript{\(\dagger\)} &
\ms{2.797\!\times\!10^{-4}}{2.31\!\times\!10^{-5}} &
\ms{1}{2.02\!\times\!10^{-7}} &
\ms{-1.52\!\times\!10^{-3}}{2.165\!\times\!10^{-3}} & -- \\
SHiPPO oracle-\(R\)\textsuperscript{\(\dagger\)} &
\ms{4.377\!\times\!10^{-3}}{9.788\!\times\!10^{-3}} &
\ms{2.44\!\times\!10^{-7}}{4.61\!\times\!10^{-7}} &
\ms{1}{4.66\!\times\!10^{-7}} & \ms{1}{0} \\
SHiPPO learned-\(R\) (true init)\textsuperscript{\(\dagger\)} &
\ms{7.019\!\times\!10^{-4}}{1.18\!\times\!10^{-3}} &
\ms{2.11\!\times\!10^{-5}}{2.86\!\times\!10^{-5}} &
\ms{1}{2.86\!\times\!10^{-5}} & \ms{1}{0} \\
SHiPPO learned-\(R\) (zero init) &
\ms{3.65\!\times\!10^{-5}}{7.95\!\times\!10^{-5}} &
\ms{1.87\!\times\!10^{-6}}{2.57\!\times\!10^{-6}} &
\ms{1}{2.60\!\times\!10^{-6}} & \ms{1}{0} \\
SHiPPO learned-\(R\) (random init) &
\ms{4.260\!\times\!10^{-4}}{5.462\!\times\!10^{-4}} &
\ms{6.58\!\times\!10^{-6}}{8.08\!\times\!10^{-6}} &
\ms{1}{8.11\!\times\!10^{-6}} & \ms{1}{0} \\
SHiPPO selective-\(R\) &
\ms{1.235\!\times\!10^{-3}}{1.713\!\times\!10^{-3}} &
\ms{5.58\!\times\!10^{-5}}{8.12\!\times\!10^{-5}} &
\ms{0.9999}{8.13\!\times\!10^{-5}} & \ms{1}{0} \\
\bottomrule
\end{tabular}
\vspace{0.3em}
\begin{minipage}{0.96\linewidth}
\footnotesize The \(R_t\!\to\!I\) intervention is applicable only to
right-transport models; ``--'' denotes not applicable.
\end{minipage}
\end{table}

\begin{table}[!htbp]
\centering
\footnotesize
\setlength{\tabcolsep}{3pt}
\renewcommand{\arraystretch}{1.08}
\caption{Transported-projection prediction auxiliary result. Entries are means \(\pm\) sample standard deviations over \(n=5\) seeds.}
\label{tab:transported-projection}
\label{tab:e2-transported-projection}
\begin{tabular}{@{}L{0.42\linewidth}cc@{}}
\toprule
Model & Eval NMSE \(\downarrow\) & Eval \(R^2\) \(\uparrow\) \\
\midrule
No-right \(r=4\) & \(1.073\!\times\!10^{-3} \pm 1.471\!\times\!10^{-4}\) & \(0.9989 \pm 1.471\!\times\!10^{-4}\) \\
SHiPPO learned-\(R\) (true init) & \(5.82\!\times\!10^{-5} \pm 6.45\!\times\!10^{-5}\) & \(0.9999 \pm 6.45\!\times\!10^{-5}\) \\
SHiPPO oracle-\(R\) & \(6.571\!\times\!10^{-4} \pm 3.808\!\times\!10^{-4}\) & \(0.9993 \pm 3.808\!\times\!10^{-4}\) \\
\bottomrule
\end{tabular}
\end{table}
\FloatBarrier

\subsection{Group-local audit and limitation}
\label{app:group-audit}

The main separation claim is made for full right transport. The
scan-compatible realization in Section~3 uses group-local right transport as a
computational restriction, so we audit grouping in the minimal paired-transport
setting. Since \(d=4\), group size 4 is equivalent to full right transport.
Group size 2 forces the effective update to use block-local slices of the
generated right action.

Table~\ref{tab:app-group-audit} reports the group-local audit; this is a
limitation diagnostic rather than part of the main full-transport separation
claim.

\begin{table}[!htbp]
\centering
\small
\setlength{\tabcolsep}{4pt}
\renewcommand{\arraystretch}{1.06}
\caption{Group-local audit. Entries use \(n=3\) seeds. Pair \(\Delta\) NMSE is
averaged over seeds; raw off-block ratio is mean \(\pm\) sample standard
 deviation. Group size 4 corresponds to full right transport for this \(d=4\)
 diagnostic.}
\label{tab:app-group-audit}
\begin{tabular}{@{}L{0.34\linewidth}C{0.12\linewidth}C{0.22\linewidth}C{0.24\linewidth}@{}}
\toprule
Model & Group &
\makecell[c]{Pair \(\Delta\)\\NMSE \(\downarrow\)} &
\makecell[c]{Raw off-block\\ratio} \\
\midrule
SHiPPO learned-\(R\) & 2 & \(1.60\!\times\!10^{-5}\) & \(0.1305 \pm 0.0277\) \\
SHiPPO learned-\(R\) & 4 & \(3.78\!\times\!10^{-5}\) & \(0 \pm 0\) \\
SHiPPO oracle-\(R\) & 2 & \(1\) & \(0.1015 \pm 0.0006\) \\
SHiPPO oracle-\(R\) & 4 & \(6.70\!\times\!10^{-8}\) & \(0 \pm 0\) \\
\bottomrule
\end{tabular}
\end{table}
\FloatBarrier

This audit should not be interpreted as evidence that group-local transport
faithfully implements full transport. Rather, it shows that learned front-end,
source, and readout maps may partially recode the diagnostic under a mismatched
grouping. The formal separation claim is therefore made for the full-transport
setting; group-local transport is a computational realization whose expressivity
depends on the grouping.

\subsection{Excluded exploratory diagnostics}
\label{app:excluded-diagnostics}

We also explored broader associative-recall diagnostics, span-capacity pilots,
and program-like wrappers. We do not use these runs as evidence for the
mechanistic separation studied here: broad associative-recall tasks can admit
shortcuts based on high-rank current-step writes, span-capacity pilots do not include
paired-difference or intervention metrics, and the current program-like wrappers
introduce additional variable-tracking difficulty beyond the transport
mechanism. These experiments motivated the paired noncommutative transport
diagnostic but are not part of the paper-facing evidence.



\section{Transport-MQAR Details}
\label{app:transport-mqar-details}

This appendix records the Transport-MQAR generator, metrics, model geometry,
training protocol, full result tables, and controller counterfactuals used for
Section~\ref{subsec:transport-mqar-main}. Transport-MQAR and the
controller-suffix intervention provide diagnostic evidence for learned
right-action pathways in an autoregressive transported-recall setting.

\subsection{Task, metrics, and generator}
\label{app:transport-mqar-task}

Transport-MQAR is a finite-field, right-transport variant of multi-query associative recall (MQAR) \citep{arora2024zoology}. Each
example is generated over \(\mathbb F_{31}\) with 256 keys and four-coordinate
values. A binding event stores a key--value pair; an operation event applies an
invertible right action to all stored values; and a query event asks for the
currently transported value of a key. The configured training length is 512, and
final evaluations use lengths \(128,512,2048,4096\). We report coordinate
accuracy, the fraction of target coordinates predicted correctly, and exact
accuracy, the fraction of queries for which all four coordinates are correct.
The main runs use a nonreducible library of 13 finite-field operations.

\begin{table}[H]
\centering
\small
\caption{Transport-MQAR generator configuration for the reported primary runs.}
\label{tab:app-transport-mqar-generator-config}
\begin{tabular}{@{}L{0.36\linewidth}L{0.58\linewidth}@{}}
\toprule
Field & Value \\
\midrule
Finite field & \(\mathbb F_{31}\) \\
Target coordinates & \(p=4\) \\
Number of keys & 256 \\
Configured training length & 512 \\
Evaluation lengths & 128, 512, 2048, 4096 \\
Generator mode & \texttt{nonreducible} \\
Event probabilities & \(p_{\mathrm{op}}=0.50,\ p_{\mathrm{bind}}=0.22,\ p_{\mathrm{query}}=0.28\) \\
Number of operation generators & 13 \\
Evaluation examples & 640 per length and seed unless noted otherwise \\
Training loss & coordinate-wise cross-entropy at query positions \\
\bottomrule
\end{tabular}
\end{table}

\subsection{Controls and realization geometry}
\label{app:transport-mqar-geometry}

GRU \citep{cho2014learning} and Transformer \citep{vaswani2017attention} are same-width sequence baselines. Free enc/dec adds static
input/output basis flexibility around a no-right SHiPPO-style memory. No-right
preserves the matrix-state geometry but fixes \(R_t=I\). DirectGen-SingleExp is
an intra-family right-action realization: it keeps the group-local matrix-state
recurrence
\[
    H_{t,g}=L_{t,g}H_{t-1,g}R_{t,g}+\widehat U_{t,g},
\]
but directly emits a group-local generator \(A_{t,g}\in\mathbb R^{P_g\times
P_g}\) and sets
\[
    R_{t,g}=\exp(\Delta_{t,g}A_{t,g}).
\]
StructGen-Split uses a structured right-generator library and a split product of
factor exponentials. Thus the DirectGen-SingleExp comparison is an intra-family
comparison between right-action realizations, and it changes both the generator
family and the discrete right-action backend.

\begin{table}[H]
\centering
\small
\caption{Model geometry and parameter counts in the primary Transport-MQAR comparison. All models use \(d_{\rm model}=128\).}
\label{tab:app-transport-mqar-model-geometry}
\resizebox{\linewidth}{!}{
\begin{tabular}{lrrrrrlr}
\toprule
Model & \(e\) & \(D_{\rm cell}\) & \(N\) & \(P_g\) & \(G\) & Role & Params \\
\midrule
GRU & -- & -- & -- & -- & -- & recurrent baseline & 0.50M \\
Transformer & -- & -- & -- & -- & -- & attention baseline & 0.89M \\
\midrule
Free enc/dec & 2 & 256 & 32 & 4 & 64 & static-basis control & 5.11M \\
No-right & 2 & 256 & 32 & 4 & 64 & right-action ablation & 4.98M \\
DirectGen-SingleExp & 1 & 128 & 32 & 4 & 32 & direct-generator right action & 1.52M \\
StructGen-Split & 2 & 256 & 32 & 4 & 64 & structured split right action & 6.03M \\
\bottomrule
\end{tabular}}
\end{table}

\subsection{Training protocol}
\label{app:transport-mqar-protocol}

Final out-of-distribution evaluation lengths are not used for model selection.
For SHiPPO-family variants, hyperparameters are selected using validation data
at the configured training length only. All main Transport-MQAR results are
means and sample standard deviations over \(n=3\) independent training seeds.
Training uses AdamW \citep{loshchilov2019decoupled} with learning rate \(5\times10^{-4}\), weight decay 0.01,
batch size 16, 5000 steps, evaluation every 250 steps, and gradient clipping at
1.0. Matched-capacity controls use a separate shared stress-test budget and are
capacity checks rather than the primary same-width comparison.

\subsection{Full primary results}
\label{app:full-transport-mqar-results}

Table~\ref{tab:app-transport-mqar-primary} reports the full length sweep behind
Table~\ref{tab:transport-mqar-summary}. The main-text interpretation is based on
this table: StructGen-Split improves coordinate accuracy relative to no-right
and static-basis ablations, while DirectGen-SingleExp remains stronger on exact
accuracy. Thus the results support the usefulness of a right-action pathway
while leaving the preferred discrete right-action realization open.

\begin{table}[H]
\centering
\small
\setlength{\tabcolsep}{4pt}
\renewcommand{\arraystretch}{1.08}
\caption{Primary Transport-MQAR length sweep. Entries are means \(\pm\) sample
standard deviations over \(n=3\) seeds.}
\label{tab:app-transport-mqar-primary}
\label{tab:app-transport-mqar-coord-primary}
\label{tab:app-transport-mqar-exact-primary}
\begin{tabular}{@{}lcccc@{}}
\toprule
\multicolumn{5}{@{}c@{}}{Coordinate accuracy} \\
\midrule
Model & 128 & 512 & 2048 & 4096 \\
\midrule
GRU & \ms{0.1841}{0.0012} & \ms{0.1150}{0.0021} & \ms{0.0982}{0.0011} & \ms{0.0952}{0.0017} \\
Transformer & \ms{0.2357}{0.0145} & \ms{0.0986}{0.0099} & \ms{0.0506}{0.0032} & \ms{0.0416}{0.0012} \\
Free enc/dec & \ms{0.1897}{0.0030} & \ms{0.1194}{0.0003} & \ms{0.0990}{0.0041} & \ms{0.0916}{0.0100} \\
No-right & \ms{0.1950}{0.0030} & \ms{0.1235}{0.0008} & \ms{0.1053}{0.0004} & \ms{0.1019}{0.0000} \\
DirectGen-SingleExp & \ms{0.1966}{0.0041} & \ms{0.1248}{0.0009} & \ms{0.1066}{0.0006} & \ms{0.1032}{0.0007} \\
StructGen-Split & \ms{0.2115}{0.0028} & \ms{0.1346}{0.0007} & \ms{0.1140}{0.0002} & \ms{0.1104}{0.0010} \\
\bottomrule
\end{tabular}

\vspace{0.75em}

\begin{tabular}{@{}lcccc@{}}
\toprule
\multicolumn{5}{@{}c@{}}{Exact accuracy} \\
\midrule
Model & 128 & 512 & 2048 & 4096 \\
\midrule
GRU & \ms{0.0535}{0.0021} & \ms{0.0239}{0.0026} & \ms{0.0167}{0.0021} & \ms{0.0159}{0.0024} \\
Transformer & \ms{0.0648}{0.0120} & \ms{0.0166}{0.0035} & \ms{0.0043}{0.0010} & \ms{0.0021}{0.0006} \\
Free enc/dec & \ms{0.0446}{0.0023} & \ms{0.0207}{0.0003} & \ms{0.0120}{0.0023} & \ms{0.0094}{0.0042} \\
No-right & \ms{0.0501}{0.0022} & \ms{0.0242}{0.0001} & \ms{0.0170}{0.0004} & \ms{0.0158}{0.0003} \\
DirectGen-SingleExp & \ms{0.0851}{0.0020} & \ms{0.0468}{0.0017} & \ms{0.0375}{0.0019} & \ms{0.0355}{0.0005} \\
StructGen-Split & \ms{0.0795}{0.0021} & \ms{0.0439}{0.0016} & \ms{0.0342}{0.0018} & \ms{0.0329}{0.0013} \\
\bottomrule
\end{tabular}
\end{table}

\subsection{Controller-suffix counterfactual}
\label{app:transport-mqar-controller-counterfactual}

For StructGen-Split, the controller emits an additional 1024-dimensional suffix
associated with the split-flow right-transport coordinates. We zero this suffix
at evaluation time while keeping all trained weights fixed. Table~\ref{tab:app-zero-extra-controller-head}
shows that the trained predictor depends on these additional coordinates. This
counterfactual does not identify the learned transport geometry or prove that
the model implements a specific noncommutative algorithm.

\begin{table}[H]
\centering
\small
\setlength{\tabcolsep}{4pt}
\renewcommand{\arraystretch}{1.08}
\caption{Evaluation-time counterfactual for the StructGen-Split controller suffix. Means over three seeds.}
\label{tab:app-zero-extra-controller-head}
\begin{tabular}{@{}lrrrrrr@{}}
\toprule
& \multicolumn{3}{c}{Coordinate accuracy} & \multicolumn{3}{c}{Exact accuracy} \\
\cmidrule(lr){2-4}\cmidrule(lr){5-7}
Length & normal & zeroed & \(\Delta\) & normal & zeroed & \(\Delta\) \\
\midrule
128  & 0.2115 & 0.1936 & -0.0179 & 0.0795 & 0.0436 & -0.0359 \\
512  & 0.1346 & 0.1250 & -0.0096 & 0.0439 & 0.0239 & -0.0200 \\
2048 & 0.1140 & 0.1074 & -0.0066 & 0.0342 & 0.0185 & -0.0157 \\
4096 & 0.1104 & 0.1044 & -0.0060 & 0.0329 & 0.0176 & -0.0154 \\
\bottomrule
\end{tabular}
\end{table}

\section{Additional Related Work and Positioning}
\label{app:related-positioning}

We organize related work by the source of the recurrent dynamics and by where channel or matrix-valued interaction enters the memory. In sequence modeling, choosing an ODE or recurrence family is itself a restrictive prior: it biases which finite-dimensional causal memory operators are representable, which stability constraints are natural, and which execution schemes are compatible with the model. Parameterization, initialization, and discretization then determine how a trainable layer explores and implements that family. SHiPPO contributes at the dynamics-selection level: it derives a channel-interacting Sylvester ODE from a transported online approximation problem, and then studies scan-compatible parameterizations of that ODE.

Relative to nearby recurrent, SSM, and matrix-state model families, SHiPPO's main distinction is that channel interaction is tied to transported online-projection semantics. It differs from ODE/CDE and stable RNN families by deriving the ODE from projection coefficients rather than choosing a generic or stability-motivated vector field \citep{chen2018neuralode,kidger2020neuralcde,chang2019antisymmetricrnn,erichson2021lipschitzrnn}. It differs from objective-derived recurrences by transporting the polynomial projection objective, yielding closed Sylvester coefficient dynamics \citep{voelker2019lmu,gu2020hippo,gu2022trainhippo,liu2025longhorn}. It differs from structured and selective SSMs by changing the memory ODE induced by the online-memory prior before choosing a scan-compatible parameterization \citep{gu2021s4,gu2022s4d,smith2023s5,gu2023mamba,dao2024transformersssms,lahoti2026mamba3}. It differs from linear attention, DeltaNet, GLA, GDN, and matrix-memory RNNs by assigning projection-coefficient and transported-channel-frame semantics to the two matrix axes \citep{schlag2021fastweights,yang2024deltanet,yang2024gla,yang2025gateddeltanet,beck2024xlstm,qin2024hgrn2,mishra2026m2rnn}. Finally, Proposition~\ref{prop:collapse} provides an internal collapse criterion: in the scan-compatible group-local realization, simultaneously reducible right transports are equivalent, after a fixed channel-basis change, to independent scalar or blockwise transported banks. This criterion connects SHiPPO to structured-transition and state-tracking work that studies expressivity and efficient execution under restricted transition families \citep{merrill2024illusion,terzic2025pdssm,siems2025deltaproduct,walker2025slice}.

\paragraph{Dynamics families as restrictive priors.}
Continuous-time sequence models use differential equations as modeling primitives. Neural ODEs parameterize the derivative of the hidden state by a neural network and compute outputs through ODE solvers, while Latent ODE and Neural CDE models extend this viewpoint to irregularly sampled or partially observed time series \citep{chen2018neuralode,rubanova2019latentode,kidger2020neuralcde}. Other recurrent architectures choose more structured ODE forms to impose stability, bounded-gradient, oscillatory, or timescale priors: AntisymmetricRNNs, coRNNs, Lipschitz RNNs, Liquid Time-constant Networks, Liquid-S4, and LinOSS are examples of this dynamics-as-prior viewpoint \citep{chang2019antisymmetricrnn,rusch2021cornn,erichson2021lipschitzrnn,hasani2021ltc,hasani2022liquids4,rusch2024linoss}. Layer-wise nonlinear SSMs and stable SSM parameterizations further show that expressivity and memory behavior depend on the chosen dynamics and its stable parameterization \citep{wang2023layerwise,wang2023stablessm}. Recent dynamical-systems frameworks also compare attention, SSMs, and RNNs in a common representation \citep{sieber2024dsf}. SHiPPO is related at the level of treating dynamics as prior, but differs in the source of the dynamics: its ODE is not a generic learned vector field or a stability-motivated continuous-time RNN, but the coefficient dynamics induced by a transported online projection problem.

\paragraph{Analytically and objectively derived memory dynamics.}
SHiPPO is closest in spirit to models whose recurrent updates are derived from an objective or analytic memory principle rather than postulated as an architecture. The Legendre Memory Unit is an important precursor: it derives a continuous-time recurrent memory cell whose ODE state represents a sliding history window in a Legendre basis \citep{voelker2019lmu}. HiPPO generalizes this projection-memory viewpoint into an online polynomial projection framework, deriving coefficient dynamics for compressing the revealed history \citep{gu2020hippo,gu2022trainhippo}. Recent work such as HiPPO Zoo revisits this line by making orthogonal-polynomial memory mechanisms explicit and interpretable \citep{goffinet2026hippozoo}. SHiPPO is complementary: it does not add a collection of explicit polynomial-basis mechanisms, but transports the channel frame of the online approximation problem itself, yielding Sylvester right-action coefficient dynamics. Other recent work derives recurrent layers from online-learning or test-time learning objectives: Longhorn views SSMs as amortized online learners, TTT layers update hidden states through self-supervised learning steps, and MesaNet derives recurrent layers from locally optimal in-context regression objectives \citep{liu2025longhorn,sun2025ttt,vonoswald2026mesanet}. SHiPPO shares the objective-derived design philosophy, but its state is a transported approximation-coefficient matrix, and its right-action term is the transport required by the chosen approximation objective.

\paragraph{Structured and selective SSMs: recurrence family versus parameterization.}
Deep SSMs illustrate that similar state-equation templates can yield different models depending on parameterization, initialization, and discretization. S4 uses a structured parameterization of continuous-time state matrices for long-sequence modeling, while DSS and S4D simplify this structure to diagonal state spaces while retaining much of the empirical benefit through appropriate parameterization and initialization \citep{gu2021s4,gupta2022dss,gu2022s4d}. LRU similarly shows that linearization, diagonalization, initialization, and normalization can recover much of the performance of deep SSMs in an RNN block \citep{orvieto2023lru}. S5 moves from many independent SISO systems to a MIMO SSM layer, and H3 introduces multiplicative interactions between SSM outputs and input projections to address recall and comparison capabilities in language modeling \citep{smith2023s5,fu2023h3}. Selective SSMs such as Mamba make SSM parameters input-dependent and design hardware-aware scans; Mamba-2/SSD clarifies the semiseparable structure behind such recurrences; Mamba-3 further modifies the recurrence through discretization, complex dynamics, and MIMO updates; and recent analyses study the role of input selectivity in approximation, memorization, and associative recall \citep{gu2023mamba,dao2024transformersssms,lahoti2026mamba3,huang2025inputselectivity}. Other work studies preferential or procedural biases within a chosen SSM family, including perturb-then-diagonalize, Hankel/Markov parameterization, spectral or frequency-bias tuning, uncertainty-aware initialization, autocorrelation-aware initialization, and mimetic initialization \citep{yu2023ptd,yu2024hope,yu2024frequency,solozabal2025spectralbias,lienen2025unhippo,liu2025autocorrelation,trockman2024mimetic}. These works modify how a chosen dynamics family is parameterized, initialized, or trained. SHiPPO instead changes the memory ODE selected by the online-memory prior before choosing a scan-compatible parameterization.

\paragraph{Matrix memories, fast weights, and gated linear recurrences.}
Many efficient recurrent models introduce channel interaction through architectural or algorithmic mechanisms. Fast-weight and delta-rule views of linear attention interpret matrix-valued recurrent states as associative memories programmed by outer-product or correction updates \citep{schlag2021fastweights,yang2024deltanet}. Gated Linear Attention and implicit-attention views of gated-linear RNNs maintain two-dimensional associative memories and introduce data-dependent gates \citep{yang2024gla,zimerman2024implicitattention}. Gated DeltaNet combines Mamba-style gating with the delta rule \citep{yang2025gateddeltanet}, and Kimi Linear introduces Kimi Delta Attention (KDA), an extension of Gated DeltaNet with finer-grained gating and efficient chunkwise execution \citep{kimiteam2025kimilinear}. Matrix-memory recurrent models such as xLSTM/mLSTM and HGRN2 increase memory capacity through exponential gating, covariance-style memory, or outer-product state expansion \citep{beck2024xlstm,qin2024hgrn2}. Recent matrix-to-matrix and weight-space RNNs further explore matrix-valued or weight-valued hidden states \citep{mishra2026m2rnn,nzoyem2025warprnn}. SHiPPO is close to these works at the level of state shape, but the axes have different semantics: the left axis indexes online approximation coefficients, while the right axis is a transported channel frame. The two-sided update is therefore not introduced as an untied matrix-state recurrence; it is the discrete realization of a transported coefficient ODE.

\paragraph{Structured transitions and state-tracking expressivity.}
A separate line of work studies which transition structures preserve expressivity while remaining parallelizable. Formal-language and state-tracking analyses identify both strengths and limitations of common SSMs, selective SSMs, diagonal or positive-eigenvalue recurrences, and gated variants \citep{sarrof2024formal,merrill2024illusion,cirone2024deepselectivessm,terzic2025regular,grazzi2024negative,shakerinava2026diagonal,alsmann2026temporallogics}. Several architectural responses enrich transition structure while trying to preserve efficient execution: fixed-point RNNs interpolate from diagonal to dense recurrences, bilinear state transitions use hidden--input multiplicative updates, adaptive unitary SSMs use input-dependent skew/unitary dynamics, DeltaProduct uses products of generalized Householder transformations, SLiCE studies structured input-dependent transition matrices, and PD-SSM uses structured sparse transitions for FSA state tracking \citep{movahedi2025fixedpoint,ebrahimi2025bilinear,karuvally2025adaptiveunitary,siems2025deltaproduct,walker2025slice,terzic2025pdssm}. Recent coefficient-dynamics views also analyze sequence models through the dynamics of coefficients used to combine past value vectors \citep{sieber2025coefficientdynamics}. SHiPPO is complementary to this literature. It does not propose a general structured-transition framework; it identifies the two-sided transition induced by transported online-memory semantics. Our collapse result shows that, in the scan-compatible group-local realization, simultaneously reducible right transports are functionally equivalent to static mixing followed by independent scalar or blockwise transported banks, motivating non-reducible right-transport families.

\paragraph{Moving-frame interpretation and non-reduction.}
For a realized transport path, SHiPPO admits an exact moving-frame HiPPO
factorization, but not an arbitrary encoder--decoder reduction: the history
encoder, coefficient decoder, metric, and Sylvester gauge term are tied by the
same transport ODE. With input-dependent controllers, this interpretation is
pathwise rather than a fixed input-independent factorization.


\end{document}